\documentclass{article} 
\usepackage{iclr2027_conference,times}

\usepackage{amsmath,amsfonts,bm}

\def\floor#1{\lfloor #1 \rfloor}
\def\1{\bm{1}}

\DeclareMathAlphabet{\mathsfit}{\encodingdefault}{\sfdefault}{m}{sl}
\SetMathAlphabet{\mathsfit}{bold}{\encodingdefault}{\sfdefault}{bx}{n}

\newcommand{\Cin}{D_{\mathrm{in}}}
\newcommand{\Cout}{D_{\mathrm{out}}}

\newcommand{\rhoout}{\rho_{\mathrm{out}}}
\newcommand{\rhoin}{\rho_{\mathrm{in}}}

\newcommand{\labeledgraphic}[4][north west]{%
  \begin{tikzpicture}
    \node[inner sep=0pt] (img) {\includegraphics[width=#3]{#4}};
    \node[anchor=#1, inner sep=2pt, font=\small\bfseries] at (img.#1) {(#2)};
  \end{tikzpicture}}

\newcommand{\labeledgraphicabove}[3]{%
  \begin{tikzpicture}
    \node[inner sep=0pt] (img) {\includegraphics[width=#2]{#3}};
    \node[anchor=south west, inner sep=2pt, font=\small\bfseries] at (img.north west) {(#1)};
  \end{tikzpicture}}

\newcommand{\multilabeledgraphic}[3]{%
  \begin{tikzpicture}
    \node[inner sep=0pt, anchor=south west] (img) at (0,0) {\includegraphics[width=#1]{#2}};
    \begin{scope}[x={(img.south east)}, y={(img.north west)}]
      \foreach \px/\py/\lbl in {#3}
        \node[anchor=north west, inner sep=2pt, font=\small\bfseries] at (\px,\py) {(\lbl)};
    \end{scope}
  \end{tikzpicture}}

\makeatletter
\newcommand{\sublabel}[2]{{\protected@edef\@currentlabel{\thefigure(#1)}\label{#2}}}
\makeatother

\usepackage{hyperref}
\usepackage{url}

\usepackage{amsthm,amssymb,amsfonts}
\usepackage{braket}
\usepackage{bbm}
\usepackage{tikz-cd}
\usepackage{enumitem}
\usepackage{mathtools}
\usepackage{siunitx}
\usepackage{nicefrac}
\usepackage{wrapfig}

\usepackage{tikz}
\usetikzlibrary{arrows.meta}

\newtheorem{definition}{Definition}

\newtheorem{lemma}{Lemma}
\newtheorem{theorem}{Theorem}

\newtheorem{corollary}{Corollary}

\newtheorem{proposition}{Proposition}

\newtheorem{example}{Example}

\newcommand{\repeatedtheoremnumber}{}
\newtheorem*{repeatedtheorem}{Theorem \repeatedtheoremnumber}

\title{Inherent Trade-offs when Enforcing \\ Correct Data Symmetries \\ in Multi-head Self-attention}

\author{T\={\i}kun \^Ong, Georg Bökman}

\iclrfinalcopy 
\begin{document}

\maketitle
\lhead{Prerint. Under review.}

\begin{abstract}
	We give a complete characterization of equivariant multi-head self-attention
	(MHSA): if an MHSA layer is equivariant to a symmetry group $G$, then $G$ can
	only act by permuting head-clusters, with QK and OV matrices
	satisfying an equivariance constraint tied to the group action. As a
	consequence, we prove that any fixed MHSA architecture that achieves exact
	equivariance by polynomially parameterizing unconstrained MHSA parameters
	inevitably leads to expressivity loss within the class of equivariant
	maps:
	the equivariance locus of unconstrained MHSA
	forms a union of extremely many Zariski-irreducible components
	in a reduced parameter space,
	and any single architecture covers at most one. For $G=D_4$ acting on $C$
	copies of the regular representation as the token feature space, we
	show that there are $\Omega(C^{64})$ components for eight attention
	heads.
	We illustrate this with experiments on artificial tasks
	and show that the predicted head-cluster structures
	can arise naturally when the attention layer is
	pressured to become equivariant.
\end{abstract}

\section{Introduction}
\label{sec:intro}

Since symmetry is present in nature, it is present in data and therefore in many machine learning tasks.
Usually, given a task we even know a priori certain symmetries that it respects.
For instance, the semantic segmentation of
an aerial or medical image does not depend on orientation
and the ground state energy of a molecule does not depend on the
spatial reference frame used to encode each atom's position.
The property of respecting a given symmetry is often formalized as equivariance,
which a model may possess either by design at the architectural level or by learning.

It may seem like enforcing a correct symmetry at the architectural level is an open goal, but this is not necessarily the case.
First of all, it may not be clear how a given equivariance can be enforced.
For common symmetry groups, such as permutations and rotations of 2D and 3D space, much effort has been devoted to designing equivariant architectures.
In the 1990's, researchers started building equivariant MLPs \citep{shawe-taylorBuildingSymmetriesFeedforward1989,wood1996representation}, followed by equivariant CNNs in the 2010's \citep{cohenGroupEquivariantConvolutional2016,weiler3DSteerableCNNs2018,kondorGeneralizationEquivarianceConvolution2018,bekkersRotoTranslationCovariantConvolutional2018,cohenGeneralTheoryEquivariant2020} and a shift towards equivariant attention and transformer architectures in the 2020's
\citep{romeroGroupEquivariantStandAlone2021,xu$E2$EquivariantVisionTransformer2023,islamPlatonicTransformersSolid2025}.

Second, even when a procedure is known for making an architecture equivariant, it may not yield all possible equivariant models available to an unconstrained baseline.
For instance, the standard way to make an MLP equivariant is to enforce an equivariance constraint on each layer, but
for ReLU-MLPs there exist special parameter configurations that give equivariant networks without each layer being equivariant, and these special configurations can parameterize different functions than layerwise equivariant networks of the same width \citep{agrawalClassification$G$invariantShallow2023}.
So in ReLU-MLPs, enforcing a correct symmetry via any layerwise equivariance construction misses out on equivariant functions available to an unconstrained baseline.

Finally, even if a simple procedure is available that yields all possible equivariant models, it can happen that one ends up with a forced choice among mutually exclusive options. This is the case for rank-bounded linear maps, where any equivariant architecture only covers one of many irreducible components of equivariant rank-bounded linear maps~\citep{kohnGeometryLinearNeural2025}.

In this paper we give a precise account of these three potential conundrums in the case of equivariant multi-head self-attention (MHSA).
First, we prove a structure theorem.
Assuming that a symmetry group $G$ acts on input and output token feature spaces $\mathbb{R}^{\Cin}, \mathbb{R}^{\Cout}$, we give a
complete characterization of maps $\mathbb{R}^{L\times \Cin}\rightarrow
	\mathbb{R}^{L\times \Cout}$ that are both $G$-equivariant and expressible as a
multi-head self-attention.
Slightly simplified, every instance of equivariant MHSA is given by a partition of the attention heads into orbits under $G$ and a joint equivariance constraint on the heads in an orbit.
Every equivariant MHSA is implementable using relatively simple weight sharing rules and equivariance constraints, so finding equivariant architectures is not an issue.
In particular, there are no ``unexpected'' equivariant parameterizations, in contrast to the ReLU-MLP case.

Based on the structure result, we use algebraic geometry to show that the space of
equivariant functions is very intricate: it is a union of extremely many
``components'', and any given architecture (way of constraining equivariance)
can cover at most one component.
Hence, an MHSA with architectural equivariance
never represents all equivariant functions representable by the
unconstrained MHSA.
Our second main result gives a combinatorial lower bound on the number of equivariant models needed to cover all equivariant components of a single unconstrained MHSA of the same size.
In general, as the feature dimension grows, the bound explodes. 

\begin{figure}[t]
    \centering
    \includegraphics[width=0.9\linewidth]{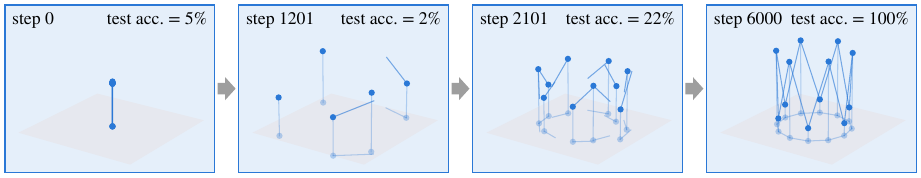}
    \caption{QK matrices form a ``crown'', training on modular addition under equivariance pressure.}
    \label{fig:crowns-from-training}
\end{figure}

\begin{wrapfigure}{r}{0.42\textwidth}
    \centering
    \vspace{-5pt}
    \includegraphics[width=0.98\linewidth]{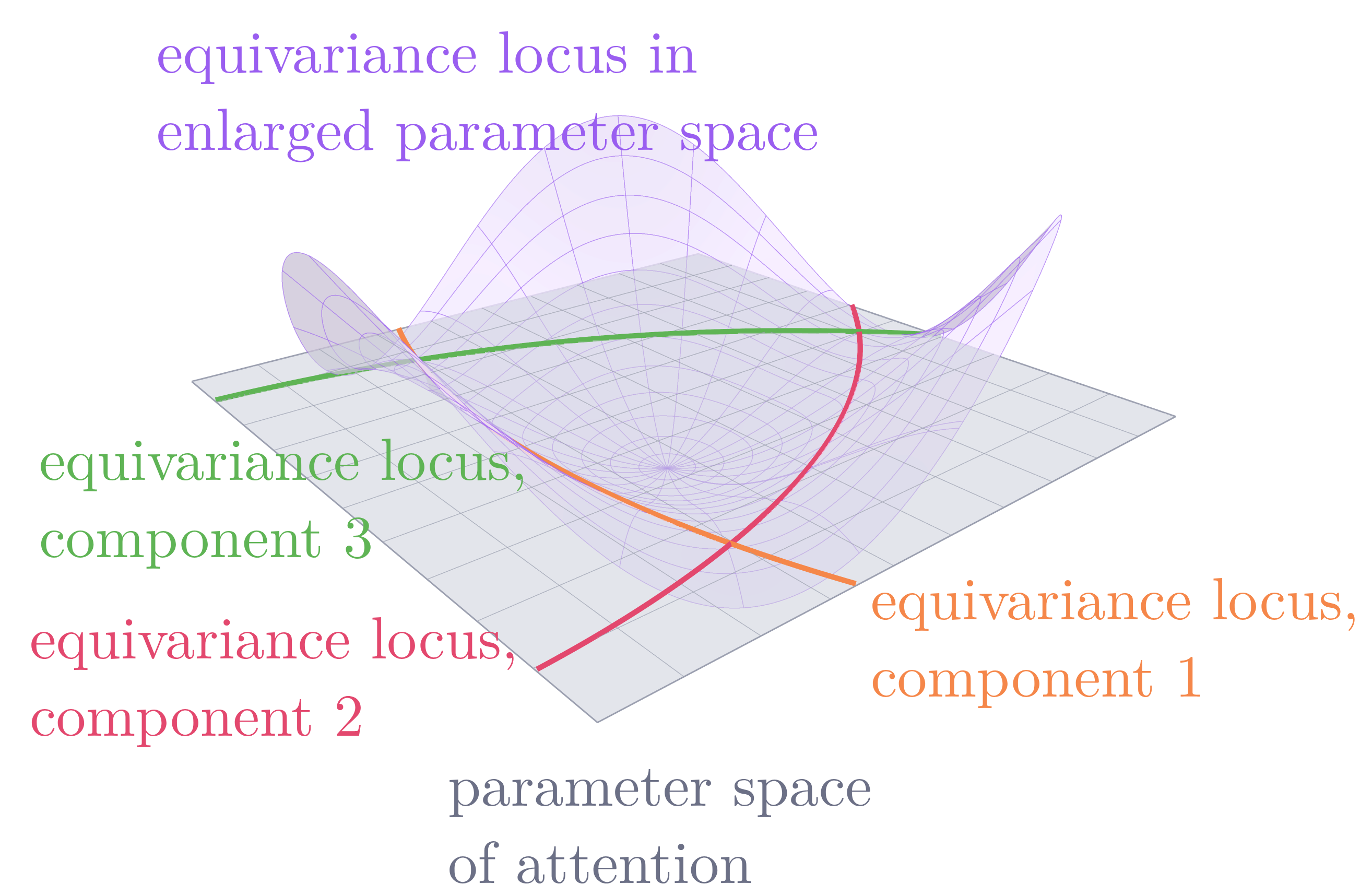}
    \caption{The equivariance dilemma.
    }
    \label{fig:placeholder2}
    \vspace{-10pt}
\end{wrapfigure}

Our results therefore show that MHSA is analogous to the rank-bounded linear map case.
There is also a further analogy with the linear maps. 
All rank-$k$ equivariant linear maps can be expressed by an equivariant linear map that is not rank constrained, which corresponds to a single ``wider'' model.
In a similar sense, we will see that all components of equivariant MHSA with $H$ heads can be expressed by a single model with $|G|H$ heads, via a group averaging construction.
The situation is illustrated in Figure~\ref{fig:placeholder2}.
The different components of the equivariance locus of a particular MHSA stem from a single equivariant component in a larger parameter space, but in the original parameter space we need to make a choice among mutually exclusive options.

Interestingly, our results together present simultaneously a case \emph{for}
hard-coded equivariance in attention and a case \emph{against}:

\emph{For (corollary of Theorem~\ref{thm:equiv-attn-char}):} 
For an MHSA layer to be equivariant, its weights must be highly structured.
There is only a finite list of constraints on the weights, each corresponding
to one component, to choose from, and any one of them results in a valid
equivariant attention layer that is easy to implement.
However, to our knowledge, only two extreme cases in this family have been implemented so far.

\emph{Against (corollary of Theorem~\ref{thm:components}):}
Since only one component can be chosen, any architecture that implements equivariant attention inevitably encodes \emph{more inductive bias than just equivariance}, relative to the unconstrained baseline with the same number of heads and the same head dimension.

To illustrate our theoretical results, we conduct two interpretability experiments: a modular addition task and a vision transformer task.
We find that in these tasks approximate equivariance is encoded structurally in MHSA, consistent with our theory for exact equivariance.
For instance, in Figure~\ref{fig:crowns-from-training} we show the emergence during training of a ``crown''-shaped orbit of size 12, consisting of QK matrices of different heads in MHSA trained on addition modulo 24.

\section{Preliminaries: Attention and Equivariance}

\textit{\textbf{ Multi-head self-attention}}. Let $\Cin, \Cout$ be the dimensions of the input and output token feature
spaces, $d$ the head dimension, and $H$ the number of heads. Usually $\Cin=\Cout
= d H$.

For weights $W_q^{(h)}, W_k^{(h)}, W_v^{(h)} \in \mathbb{R}^{d\times \Cin}, W_o^{(h)}\in \mathbb{R}^{\Cout\times d}$, 
where $1\le h\le H$, the multi-head
self-attention on $L$ input tokens $x = (x_1, x_2, \dotsc, x_L) \in \mathbb{R}^{L\times \Cin}$
is given by\footnote{The conventional temperature scaling in the QK matrices $M_h$ can be absorbed by scaling
$W_q$ and $W_k$ and does not affect our theorem statements and proofs, so we simply
leave it out in the definition of attention.}
\begin{equation}
	\label{eq:attn-def}
	\mathrm{MHSA}(x; W_q, W_k, W_v, W_o)_i =
	\sum_{h=1}^H W_o^{(h)}\frac{
        \sum_{j=1}^L
            \exp\left((W_q^{(h)} x_i)^\top (W_k^{(h)} x_j)\right)
            W_v^{(h)}x_j
    }
    {\sum_{j=1}^L\exp \left((W_q^{(h)}x_i)^\top (W_k^{(h)} x_j)\right)}.
\end{equation}
Eq.~\eqref{eq:attn-def} depends only on $M_h:= W_q^{(h)\top} W_k^{(h)}$ and 
$R_h := W_o^{(h)} W_v^{(h)}$,
the QK and OV matrices in the terminology of \citet{elhageMathematicalFrameworkTransformer2021}.
\textit{\textbf{ Equivariance}}. Symmetry is commonly present in both real-world tasks as well as 
artificial tasks used in interpretability experiments.
The mathematical formalization of ``model respecting symmetry''
is equivariance, which we motivate with two concrete examples in the following, both in the context of attention/transformers.

\begin{example}[Modular addition \citep{nandaProgressMeasuresGrokking2023}]
Let $p$ be a positive integer. Suppose we want to learn
modular addition, i.e., 
the function 
$(a, b)\mapsto a + b \pmod p$.
The inputs $(a, b)$ are one-hot encoded as $(e_a, e_b)$. For example, $e_2 = (0, 0,
1, 0)^\top$ for $p=4$. 
The model $T$ is supposed to output $p$ logits, the argmax of which should
be $a + b \pmod p$.
The set $C_p := \{0, 1, \dotsc, p-1\}$ is an abelian group under modular
addition. It acts on $\mathbb{R}^p$ by permuting the basis vectors, which are
just the one-hot encodings: $\rhoin(g) e_a = e_{g+a}$ where $g\in C_p$. For example, if $p=3$ and $g=1$,
then $\rhoin(g)$ is the matrix $\begin{psmallmatrix}0 & 0 & 1 \\ 1 & 0 & 0 \\ 0 & 1 & 0\end{psmallmatrix}$.

Modular addition satisfies $(g+a) + (g+b) \equiv 2g+a+b \pmod p$, so 
the model should satisfy the following corresponding condition:
\begin{equation}
T(\rhoin(g)e_{a}, \rhoin(g)e_{b}) = \rhoout(g)T(e_a, e_b), \quad \text{with $\rhoout(g)=\rhoin(g)^2$.}
\end{equation}
\end{example}

\begin{example}[Semantic segmentation]
An input image $u \in \mathbb{R}^{C\times N \times N}$ is divided into $L$ square patches,
viewed as $L$ tokens $(x_1, \dotsc, x_L)\in \mathbb{R}^{L\times CP^2}$, where $P
:= N / \sqrt{L}$ is the number of pixels on one side of each patch.
The dihedral group $D_4$ (rotations by $\ang{90}$ and reflections) acts on the image $u$,
rotating each patch and permuting the patches simultaneously. 
For each $g\in D_4$, let $\rhoin(g) \in \mathbb{R}^{CP^2\times CP^2}$ be the permutation matrix that 
rotates or reflects the $P\times P$ pixels without touching the channel index.
The model's task is to output a pixel-level segmentation mask $m \in
\mathbb{R}^{L\times DP^2}$, where $m_i\in \mathbb{R}^{DP^2}$ is the $P\times P$ array of $D$ class
logits of pixels in patch $i$.

We expect ``rotate image $\rightarrow$ segment'' to yield the same result as ``segment $\rightarrow$ rotate mask'' for 
data with rotation symmetry.
In other words, if $(m_1', \dotsc, m_L') = T(\rhoin(g)x_{g^{-1}.1}, \dotsc, \rhoin(g)x_{g^{-1}.L})$ is the output of the model on the rotated image, we expect  
$m_i' = \rhoout(g) m_{g^{-1}.i}$, where $\rhoout(g) \in \mathbb{R}^{DP^2\times DP^2}$ rotates the mask at patch $i$.
Because attention (and transformers in general) is blind to the order of the
tokens, this condition is the same as
\begin{equation}
T(\rhoin(g)x_1, \dotsc, \rhoin(g)x_L)_i = \rhoout(g)T(x_1, \dotsc, x_L)_i.
\end{equation}



\end{example}

In each of the preceding examples, we concluded that it is natural for the
model $T$ to satisfy $T(\rhoin(g)x_1, \dotsc, \rhoin(g)x_L)_i = \rhoout(g)T(x_1,
\dotsc x_L)_i$, with different $\rhoin$ and $\rhoout$ in each case.
If this holds, we say that $T$ is \emph{equivariant at token $i$}.
For next-token prediction, this is the natural condition with $i=L$.
For semantic segmentation, one expects equivariance at all tokens.

\section{The Structure Theorem for Equivariant MHSA}

In this section, we present Theorem~\ref{thm:equiv-attn-char}, which is the first of our two main results.
It completely characterizes the set of parameters, i.e.,
tuples $W = (W_q, W_k, W_v, W_o)$ where  $W_q, W_k, W_v\in \mathbb{R}^{H\times d\times \Cin}$
and $W_o\in \mathbb{R}^{H\times d\times \Cout}$,
for which $x\mapsto \mathrm{MHSA}(x; W)$ is $G$-equivariant with
respect to the token-wise $G$-action. That is,
\begin{equation}
	\label{eq:rmhsa-equiv-condition}
	\mathrm{MHSA}\big(\rhoin(g)x_1, \dotsc, \rhoin(g)x_L; W\big)_i
	=  \rhoout(g)\mathrm{MHSA}\big(x_1, \dotsc, x_L; W\big)_i\quad \forall g\in G.
\end{equation}


\begin{example}
	Inspecting Eq.~\eqref{eq:attn-def}, we see that
	if $W_q^{(h)}, W_k^{(h)}, W_v^{(h)}, W_o^{(h)}$ all commute with $G$ for each head $h$,
	then $\mathrm{MHSA}(-; W_q, W_k, W_v, W_o)$ is $G$-equivariant.
	However, according to the following theorem, this is not the only possibility.
\end{example}

\begin{theorem}[Structure of equivariant MHSA]
	\label{thm:equiv-attn-char}
	Let $G$ be a group acting on $\mathbb{R}^{\Cin}$ and $\mathbb{R}^{\Cout}$
	orthogonally\footnote{
		That is, $\rhoin(g)^\top = \rhoin(g)^{-1}$ and $\rhoin(g)\rhoin(g') = \rhoin(gg')$. Same for $\rhoout$.}.
	Assume the number of tokens $L$ is at least three.
	Let $W_q, W_k, W_v, W_o$ be the weights of an $H$-head self-attention. 
	Let $M_h:= W_q^{(h)\top} W_k^{(h)}  \in \mathbb{R}^{\Cin\times\Cin}$ and $R_h := W_v^{(h)} W_o^{(h)} \in \mathbb{R}^{ \Cout\times \Cin}$ be the QK and OV matrices of head $h$.
	For any matrix $m\in \mathbb{R}^{\Cin \times \Cin}$, define $R_m := \sum_{h: M_h = m} R_h \in \mathbb{R}^{\Cout\times \Cin}$. This is the total OV matrix corresponding to one QK matrix. If no $M_h$ is equal to $m$, then $R_m=0$ by definition.

	Let $T:= \mathrm{MHSA}(-; W_q, W_k, W_v, W_o): \mathbb{R}^{L\times\Cin}\rightarrow \mathbb{R}^{L\times \Cout}$
	denote the function represented by the multi-head self-attention layer with the given weights.

	The following are equivalent:
	\vspace{-.4em}
	\begin{itemize}
		\setlength{\itemsep}{1pt}
		\setlength{\parskip}{0pt}
		\setlength{\parsep}{0pt}
		\item[(1)] $T$ is equivariant at one output token.
		\item[(2)] $T$ is equivariant at all output tokens.
		\item[(3)] $R_{\rhoin(g)m\rhoin(g)^{-1}} = \rhoout(g) R_m \rhoin(g)^{-1}$ for all $m\in \mathbb{R}^{\Cin\times \Cin}$
		      and $g\in G$.
	\end{itemize}
\end{theorem}

(1)/(2) $\Rightarrow$ (3) is the non-obvious direction. The proof is found in Appendix~\ref{app:proof}.

For example, suppose that all heads have distinct QK matrices, i.e., $M_h \neq M_{h'}$ for $h\neq h'$,
and suppose, additionally, that no head's OV matrix vanishes, i.e., $R_h \neq 0$ for all $h$. Then
Theorem~\ref{thm:equiv-attn-char} says that $T$ is $G$-equivariant if and only if
for each QK matrix $M_h$ and any group element $g\in G$, there is some head $h'$ ($h'=h$ \emph{is} allowed)
such that $M_{h'} = \rhoin(g) M_h \rhoin(g)^{-1}$ and $R_{h'} = \rhoout(g)R_h\rhoin(g)^{-1}$.

In general, the QK matrices of different heads $h\neq h'$ might coincide ($M_h =
	M_{h'}$). If a set of heads $S\subset \{1, \dotsb, H\}$ share one QK matrix
$M_S$, and $M_h \neq M_S$ for all $h\notin S$, we say $S$ forms a \textit{head
	cluster}. Thus, the set of all heads $\{1, \dotsb, H\}$ is partitioned into
disjoint head clusters. For each cluster $S$, the total OV matrix is $R_S := R_{M_S} = \sum_{h\in S} R_h$
(note that the symbol $R$ is now triply overloaded: the subscript can be a head, a $\Cin\times \Cin$ matrix, or a head cluster).
If $R_S=0$, we say the cluster $S$ is \textit{dead}.

\begin{wrapfigure}{l}{.53\textwidth}
	\centering
	\input{contents/appendix/clustering-diagrams_5.tex}
	\caption{All $33$ partial $C_6$-clusterings of three heads and the equivariance constraints corresponding to the first two.
		Empty dots are dead heads.
	}
	\label{fig:c6-clustering}
\end{wrapfigure}

We can then rephrase Theorem~\ref{thm:equiv-attn-char} in the following way:
there is a $G$-action $\sigma$ on the non-dead head clusters such that
$M_{\sigma(g)S} = \rhoin(g) M_S \rhoin(g)^{-1}$ and $R_{\sigma(g)S} = \rhoout(g)
	R_S \rhoin(g)^{-1}$. We will call such an arrangement (head clusters,
which clusters are dead, and a $G$-action on the non-dead clusters)
a \textit{partial $G$-clustering}. We will often use $X, X_0$ to denote the sets
of non-dead and dead clusters respectively ($X:= \{C: R_C\neq 0\}, X_0:= \{C :
	R_C = 0\}$). Note that $X$ and $X_0$ are sets of sets of heads.

\begin{example}
	\label{ex:trivial-clustering}
	The partial $G$-clustering
	where all clusters consist of exactly one head, no cluster is dead, and the
	$G$-action is trivial, is called the \textbf{trivial ($G$-)clustering}.
\end{example}

The weaker version of Theorem~\ref{thm:equiv-attn-char} when the QK matrices
$M_h$ are pairwise distinct can be obtained by combining the generic
identifiability result of \cite{henry_generic_2026} with the framework of
\mbox{\cite{shahverdiIdentifiableEquivariantNetworks2026}}, which builds on
\cite{marchettiHarmonicsLearningUniversal2024}.
Theorem~\ref{thm:equiv-attn-char} is stronger in that it is a global statement:
it also describes nongeneric parameter configurations in which several heads
share the same QK matrix. In this case, the individual effective OV matrices
$R_h$ are not identifiable; only their sums over the corresponding fibers are
intrinsic, as noted by \cite{cirrincioneRankHeadChannelNonIdentifiability2026},
and required to respect equivariance. The global form is crucial to the proof of
Theorem~\ref{thm:components}: it gives a Zariski-closed cover of the
equivariance locus (see Section~\ref{app:irrcomp-proof}), and having control of
every member of the cover is necessary for proving maximality of the counted
irreducible components.


\begin{example}
	Let $G= C_2 = \{e, r\}$ with two
	attention heads $H = \{1, 2\}$.
	If $M_1\neq M_2$ and $X_0 = \varnothing$, then $\chi$ is bijective and $G$
	acts on $H$.
	There are two possibilities:
	either (a) the action is trivial or (b) $\sigma(r)(1) = 2$ and $\sigma(r)(2) =1$.
	In case (a), $\rhoout(r)R_h\rhoin(r)^{-1} = R_h$ and $\rhoin(r)M_h\rhoin(r)^{-1}$ commute with $\rho(r)$ for both $h\in H$.
	In case (b), $R_1$ and $R_2$ are related by $\rhoout(r)R_1 \rhoin(r)^{-1} = R_2$, and similarly
	$\rhoin(r)M_1\rhoin(r)^{-1} = M_2$.
\end{example}

\begin{example}
	For $H=1$, if the OV matrix is not zero, the equivariance
	conditions simply become that both $M$ and $R$ are $G$-equivariant, recovering
	Theorem~1 of
	\cite{ongUnifiedFrameworkVision2026}.
\end{example}

\begin{example}
	Consider an arbitrary finite group $G$.
	Take $X = I\times G$, where $I$ is some finite index set.
	Let $G$ act on $X$ by left
	multiplication on the second factor, i.e., $\sigma(g)(j, g') = (j, gg')$.
	In this case, the equivariance condition reads
	$gM_{j, g'}g^{-1} = M_{j, gg'}$ and similarly for the $R_{j, g}$.
	In particular, this means $M_{j,g} = gM_{j,e}g^{-1}$ and $R_{j,g} = gR_{j,e} g^{-1}$, with
	$M_{j,e}$ and $R_{j,e}$ unconstrained.
	This is the construction used by the Platonic Transformer \citep{islamPlatonicTransformersSolid2025}.
\end{example}

For a Lie group $G$, the connected component $G_0$ of the identity must act
trivially on the non-dead head clusters. 
That is, we equivalently have a $\mathfrak{C}$-action, where $\mathfrak{C}:= G/G_0$ is the
group of connected components.
Therefore, if $G$ is connected, each head must be individually equivariant.

\section{A Lower Bound on the Number of Equivariant Models}
\label{sec:implications}

We now discuss our second main result, built on Theorem~\ref{thm:equiv-attn-char}, which is,
roughly speaking: ``\textit{An unconstrained MHSA can represent many more
	$G$-equivariant functions than a polynomially parameterized MHSA with exact
	$G$-equivariance.}'' We will quantify this lack of expressivity by counting how
many polynomially parameterized models are needed in order to represent all
$G$-equivariant MHSAs. This motivates the following definition:

\begin{definition}
	Let $\Lambda=\mathbb{R}^P$ be some parameter space, and let
	$\beta: \Lambda\rightarrow \mathbb{R}^{H\times d\times \Cin}\times \mathbb{R}^{H\times d \times \Cin}\times \mathbb{R}^{H\times d \times \Cin} \times \mathbb{R}^{H\times \Cout\times d}$
	be a polynomial parameterization map. 
	We define the
	\textbf{($\beta$-)polynomially parameterized MHSA} to be the parameterized map
	\begin{equation}
		\begin{aligned}
			\beta\text{\rm-MHSA}: (\mathbb{R}^{L\times \Cin}) \times  \Lambda & \rightarrow (\mathbb{R}^{L\times \Cout})      \\
			(x, \lambda)
			                                                               & \mapsto \mathrm{MHSA}(x; \beta(\lambda)).
		\end{aligned}
	\end{equation}
    Given $G$-representations $\rhoin$, $\rhoout$ on $\mathbb{R}^{\Cin}$ and $\mathbb{R}^{\Cout}$ respectively,
    a family $\mathcal{M} = (\beta\text{\rm-MHSA})_{\beta\in \mathbf{B}}$
	is \textbf{$G$-complete} if each $\beta\text{\rm-MHSA}\in \mathcal{M}$
	realizes only $G$-equivariant maps
	and any $G$-equivariant map expressible by 
	MHSA with $H$ $d$-dimensional heads can be expressed by $\beta\text{\rm-MHSA}$ for some
	$\beta\in\mathbf{B}$.
\end{definition}
For a $G$-representation $\rho: G\rightarrow \mathbb{R}^{D\times D}$, let
$c_{G,d}(\rho)$ denote the number of maximal (with respect to containment) subrepresentations $V\subset
\mathbb{R}^D$ such that $0 < \dim V \le d$ up to equivalence. 
For example, if $G= C_4$ acts on $\mathbb{R}^4$ by one copy of each of $\mathrm{A}_1$, $\mathrm{A}_2$, and $\mathrm{E}_1$, 
there are two maximal subrepresentations with $0 < \dim V \le 2$: $\mathrm{A}_1\oplus \mathrm{A}_2$ and $\mathrm{E}_1$,
so $c_{C_4, 2}(\mathrm{A}_1\oplus \mathrm{A}_2 \oplus \mathrm{E}_1) = 2$.

\begin{theorem}
	\label{thm:components}
	  Let $G$ be finite.
    Assume $\Cin = \Cout$ and $\rhoin = \rhoout$
    (that is, the $G$-actions on input and output token features coincide).
	A $G$-complete family of polynomially parameterized MHSA
	with $H$ heads and head dimension $d$
	necessarily contains at least
	$\frac{1}{H!}c_{G,d}(\rho)^{2H}$
	elements (models).
\end{theorem}

The proof (Appendix~\ref{app:irrcomp-proof}) involves understanding the
Zariski-irreducible components of the equivariance locus in a suitable parameter
space (that of RMHSA, see Appendix~\ref{app:einleitung}), which we consider to be part of the neuroalgebraic geometry program
\citep{marchettiAlgebraUnveilsDeep2025}. These components have the same
algebro-geometric origin as the ones in equivariant linear networks
\citep{kohnGeometryLinearNeural2025}. The lower bound is obtained by counting
only components corresponding to the trivial action on the heads (trivial
$G$-clusterings), and including nontrivial actions could only increase the
bound.

\begin{example}
	Let $G$ be the trivial group acting on $\mathbb{R}^D$, $C >0$.
	With $d = \dim V / |H| = C/|H|$, we have
	$c_{G, d}(V) = 1$, and the lower bound is $1/H!$, which is vacuous.
    Indeed, for a trivial group, any function is equivariant, so the equivariance locus is 
    the entire parameter space, which can be easily covered by one parameterization.
\end{example}

\begin{example}\label{ex:d4-ants}
	Take $G=D_4$ acting on $\mathbb{R}^{8C}$ via $C$ copies of the regular
	representation $\rho_{\mathrm{reg}}$. 
    With eight attention heads of dimension $C$ ,
    the bound is $8!^{-1} c_{D_4, C}(\rho_{\mathrm{reg}}^{\oplus C})^{16}$.
    For $C=96$ (total token dimension $=8C = 768$), this number is about $1.6\times
	10^{96}$ (derivation in Appendix~\ref{app:d4_components_bound}). That is roughly enough atoms for each ant on Earth to get its own copy of the observable universe.
    
\end{example}

\begin{wrapfigure}{r}{0.5\textwidth}
\centering
\vspace{-1em}
\includegraphics[width=.5\textwidth]{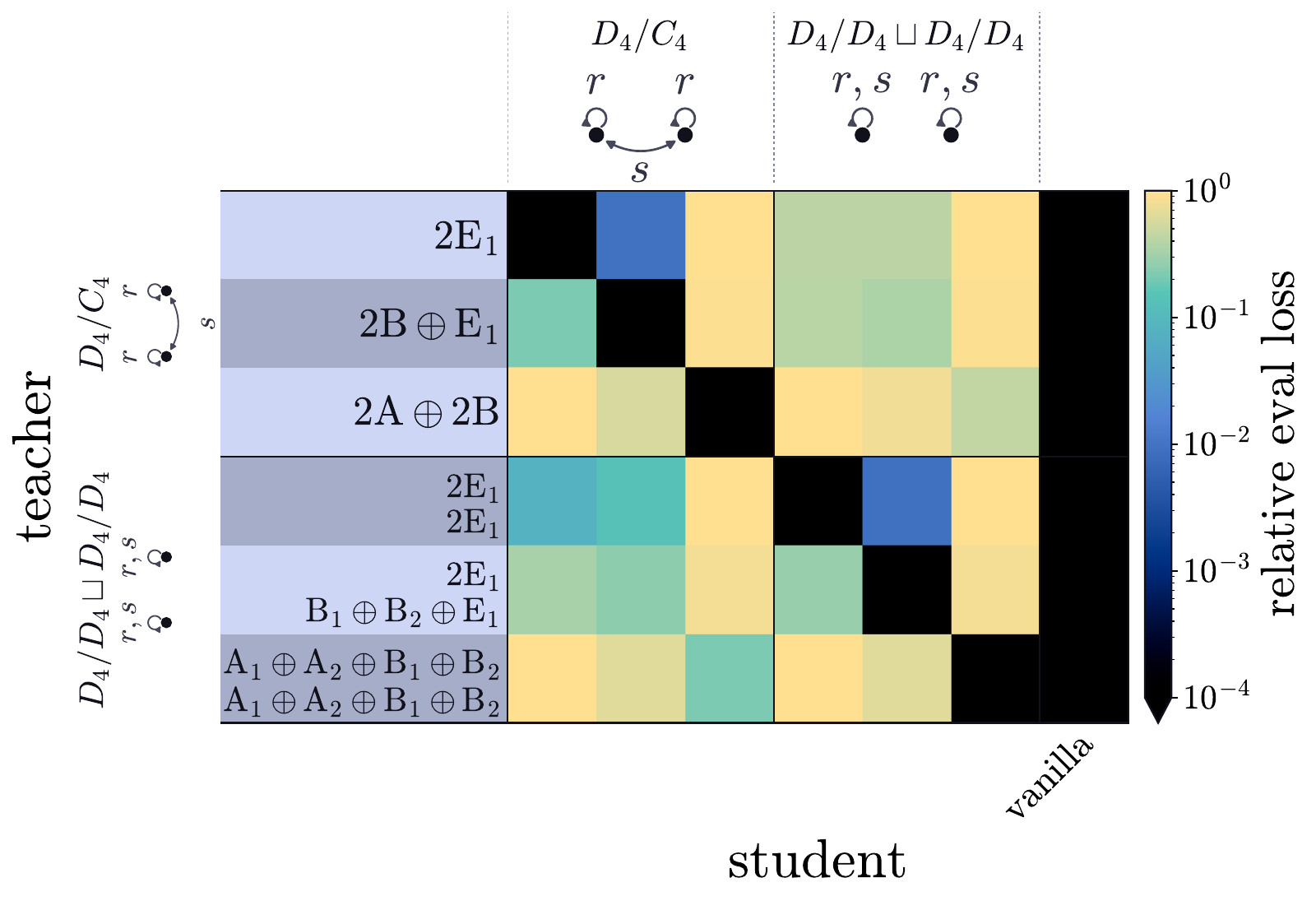}
\vspace{-2em}
\caption{
Teacher-student experiments among six 
possible $D_4$-equivariant architectures. The last column is an unconstrained attention student.
See Appendix~\ref{app:exp-ts} for more details.
}
\label{fig:teacher-student}
\end{wrapfigure}

How different are these components? For 
$G=D_4$ acting on $\mathbb{R}^8$ via the regular representation
with $H=2$ heads and head dimension $d=4$, we take six of these components,
giving six possible $D_4$-equivariant multi-head self-attention layers.
These are indexed by (1) a $D_4$-action on the two heads, and (2) 
for each $D_4$-orbit, a pair (one for OV and one for QK) of four-dimensional subrepresentations of $\mathbb{R}^8$ restricted to the stabilizer. More intrinsically, each  is labeled by a pair of $D_4$-equivariant vector bundles over the same base space, see Appendix~\ref{app:equiv-bundle}. 
We consider here two $D_4$ actions with three subrepresentation choices for each, as shown in Fig.~\ref{fig:teacher-student}.
We randomly initialize each of these six equivariant models as the teacher and train all six models as well as the unconstrained MHSA against it with MSE loss.
The results are shown in Fig.~\ref{fig:teacher-student}'s heatmap:
the unconstrained MHSA can learn to be equivariant in all six ways, but each equivariant model is only able to learn anything when the teacher 
is in the same component, save for two exceptions when the student has all irreps the teacher has: for example, with $D_4/C_4$,
the $2B\oplus E_1$ model has a relatively low loss against the $2E_1$ teacher.

Theorem~\ref{thm:components} shows that it is generally impossible to match the equivariant expressivity of an unconstrained MHSA using a single equivariant model of the same size.
However, given $H,d$, it is possible to construct larger equivariant ``covering models'', which can express all equivariant MHSA with $H$ heads and head dimension $d$.
For finite $G$, we can use group averaging to construct a covering model with $H'=H|G|$ heads and head dimension $d'=d$.
The group average of $F:\mathbb{R}^{D_\mathrm{in}}\to\mathbb{R}^{D_\mathrm{out}}$ is
$ \bar F(x) = {\tfrac{1}{|G|}}\sum_{g\in G} \rhoout(g) F(\rhoin(g)^{-1} x)$,
which is always equivariant and satisfies $\bar F=F$ for equivariant $F$.
Given unconstrained weight matrices $W_q^{(h)}, W_k^{(h)}, W_v^{(h)}, W_o^{(h)}$, we obtain the group average of the unconstrained MHSA by setting
\begin{equation*}
    W_t^{(h,g)} = W_t^{(h)} \rhoin(g)^{-1}, ~\text{for $t=q,k,v$,}\qquad
    W_o^{(h,g)} = \tfrac{1}{|G|}\rhoout(g) W_o^{(h)}.
\end{equation*}
In particular, this covering model has precisely as many parameters as the unconstrained MHSA.
Thus we get an upper bound on the number of heads needed for a covering model as $H|G|$.
It would be interesting to find minimal covering models for practically useful cases such as Example~\ref{ex:d4-ants}.

\section{Experiments}
\label{sec:exps}

\subsection{Modular Addition}
It has been shown \citep{nandaProgressMeasuresGrokking2023} that a single
transformer block trained on modular addition can learn to perform the discrete
Fourier transform over $C_p$, the cyclic group of order $p$.
Here, we
train a single unconstrained $32$-head self-attention layer with head dimension
$d=8$ to learn addition of three tokens $(a, b, c)\mapsto (b+c, c+a, a+b)$ modulo $p=24$ or $p=15$.
In addition to the task loss, we also include an explicit equivariance loss
$l_{\text{eq}}$. See Appendix~\ref{app:exp-modadd} for details.

We briefly summarize the representation theory of $C_p$ \footnote{which is just discrete Fourier transform.}, which is used to visualize the QK and OV matrices $M_h, R_h \in \mathbb{R}^{p\times p} = \mathbb{R}^p\otimes \mathbb{R}^p$:
the irreps of $C_p$ are $U_0$ (trivial, 1D, frequency $0$), $U_\omega$ (2D, frequency $\omega$) for $1\le \omega < p/2$,
and $U_{p/2}$ (sign, 1D, frequency $p/2$, only if $p$ is even).
The tensor product of two irreps corresponds to taking sum and difference of the two frequencies. For example,
$U_2\otimes U_2 \cong 2U_0 \oplus U_4$ (if $p\ge 9$).
The regular representation decomposes as $\mathbb{R}^{p} \cong \bigoplus_{\omega=0}^{\floor{\frac{p}{2}}}U_\omega$.

For a pair (query frequency $\omega_q$, key frequency $\omega_k$), pick a plane
in $U_{\omega_q}\otimes U_{\omega_k}\subset \mathbb{R}^{p}\otimes
	\mathbb{R}^{p}$ isomorphic to $U_{\omega}$, where $\omega = \omega_q + \omega_k$
or $\omega_q - \omega_k$ (note that the frequencies are defined modulo $p$ and signs). We
project each $M_h$ onto this plane for visualization.
For the OV matrices, the representation $\rhoout$ on the output is doubled.
For odd $p$ ($=15$), we still label the irrep subspaces in the output using the original frequency,
so ``frequency $\omega$'' in the output space refers to the subspace isomorphic to $U_{2\omega}$.
Thus, for a pair (input frequency $\omega_i$, output frequency $\omega_o$),
we pick an irrep plane in $U_{\omega_q}\otimes U_{2\omega_o}$ to project the OV matrices on.

\begin{figure}[htbp]
	\centering
	\includegraphics[width=.88\textwidth]{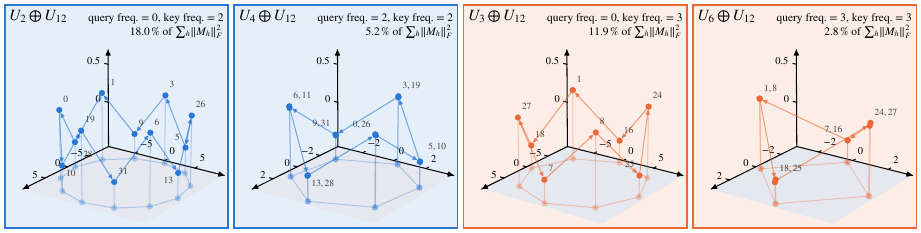}
	\caption{Learned QK matrices form ``crowns''
		in 3-dimensional subspaces of $\mathbb{R}^{p\times p}$
		isomorphic to $U_{\omega}\oplus U_{12}$, where $\omega$ is the frequency of a two-dimensional irrep.
		We omit heads that have small components in the $U_\omega$ plane.
		The numbers next to the points are the head indices.
		Blue panels: $C_{24}/C_2$-orbit formed by heads $0, 10, 19, 31, 6, 13, 26, 5, 3, 9, 11, 28$.
		Orange panels: $C_{24}/C_3$-orbit formed by heads $7, 8, 25, 24, 16, 1, 18, 27$.
		The arrow attached to a dot points to where it is carried by the action of the generator of $C_{24}$.
	}
	\label{fig:modadd_p24_3d}
\end{figure}

For $p=24$, we additionally project each $M_h$
onto $U_0 \otimes U_{12} \cong U_{12}$.
The generator of $C_{24}$ acts on the three-dimensional space $U_\omega\oplus U_{12}$ by
rotating the $U_{\omega}$ plane by $2\pi\cdot \frac{\omega}{24}$
and simultaneously flipping the $U_{12}$ axis.
Hence, if $\omega$ divides $24$, Theorem~\ref{thm:equiv-attn-char} predicts that
the QK matrices organize themselves in this 3D subspace into regular $\frac{24}{\omega}$-gons in the $U_\omega$ plane with alternating components in $U_{12}$.
This is indeed what we find in the experiment (Fig.~\ref{fig:modadd_p24_3d})


\begin{figure}[htbp]
	\centering
	\includegraphics[width=.9\textwidth]{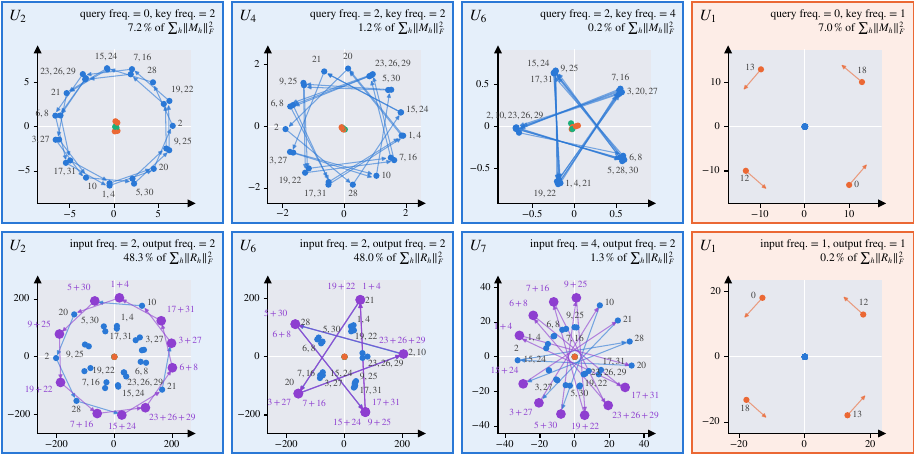}
	\caption{Learned QK and OV matrices in the $p=15$ run.
		Top row: QK matrices projected onto 2D irrep planes.
		Bottom row: OV matrices.
		Blue panels: $26$ heads form an orbit of size $15$, with one $3$-head cluster sharing one QK matrix,
		and $9$ two-head clusters.
		In the plots of OV matrices, we sum the $R_h$ within clusters, and the results are shown as purple dots.
		In the each of the two $U_2$ planes, the $M_h$ and the cluster-summed $R_h$ form
		regular $15$-gons. This is because the rotation generator is mapped to rotation by $\nicefrac{2}{15}\cdot 2\pi$
		in $U_2$ (as indicated by the arrows), and $2$ is coprime to $15$.
		The same applies to frequencies $4$ and $7$.
		In frequency $6$, however, the generator rotates the plane by $ \nicefrac{6}{15}\cdot 2\pi = \nicefrac{2}{5}\cdot 2\pi$,
		so the 15-gons collapse into pentagrams.
	}
	\label{fig:modadd_p15}
\end{figure}

The $p=15$ run exhibits a behavior not present in the $p=24$ run.
In Figure~\ref{fig:modadd_p15}, we project the QK and OV matrices onto irrep planes with
various frequencies. In the top-left plot ($U_2$), we see $26$ heads forming
a regular $15$-gon orbit. This is fine because three of them clustered together, and
another $18$ of them paired into nine two-head clusters.
This is then reflected in the OV plot (bottom left, also $U_2$):
the individual $R_h$ (blue dots) themselves do not form a size-$15$ orbit of $C_{15}$, but the
cluster-summed values (purple dots) do. The same discussion carries over to the four other blue panels,
where the same $26$ heads are arranged differently due to different frequencies.

That still leaves us with $32 - 26 = 6$ heads. Heads $11$ and $14$ are singleton
clusters, i.e., they are individually equivariant and they only every appear in
the center of each plot in Fig.~\ref{fig:modadd_p15}.
The remaining four form a ``pseudo-orbit'' (orange panels in Fig.~\ref{fig:modadd_p15}).
Since $4$ does not divide $15$, it is impossible for these four heads to form a $C_{15}$-orbit.
Instead, they form a square in the chosen $U_1$ planes, an arrangement that breaks equiviariance,
but not too much since their total squared Frobenius norm is only about $0.48\%$
of the total of all heads
(the final equivariance loss is $\sim 10^{-4}$).

We would like to note a caveat:
Since the input and output are one-hot encodings, solving the task does not
necessarily imply equivariance on the entire $\mathbb{R}^{p}$,
so an equivariance loss is used to explicitly encourage the model to become equivariant.
In our case, using solely $l_{\text{task}}$ resulted in the model learning the
task without becoming equivariant. On the other hand, $l_{\text{eq}}$ alone
often results in collapse of the OV matrices. So both losses are necessary for
the structure.
Moreover, both models reach $100\%$ training accuracy, but
the $p=24$ model generalizes perfectly (test accuracy $100\%$),
while the $p=15$ model does not (test accuracy $0\%$),
suggesting that the intricate structures come entirely from equivariance, and not the task.
This is consistent with Theorem~\ref{thm:equiv-attn-char}.

\subsection{Almost Equivariant Vision Transformers}
\label{sec:equivvf}

\begin{figure}[t]
	\centering
	\begin{minipage}[c]{.8\textwidth}
		\centering
		\labeledgraphic{a}{\linewidth}{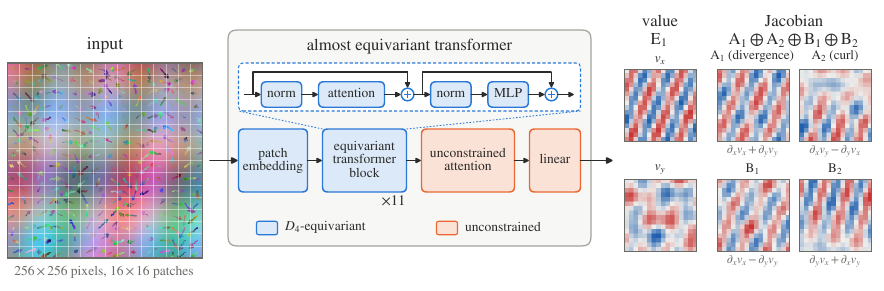}
		\par\medskip
		\multilabeledgraphic{\linewidth}{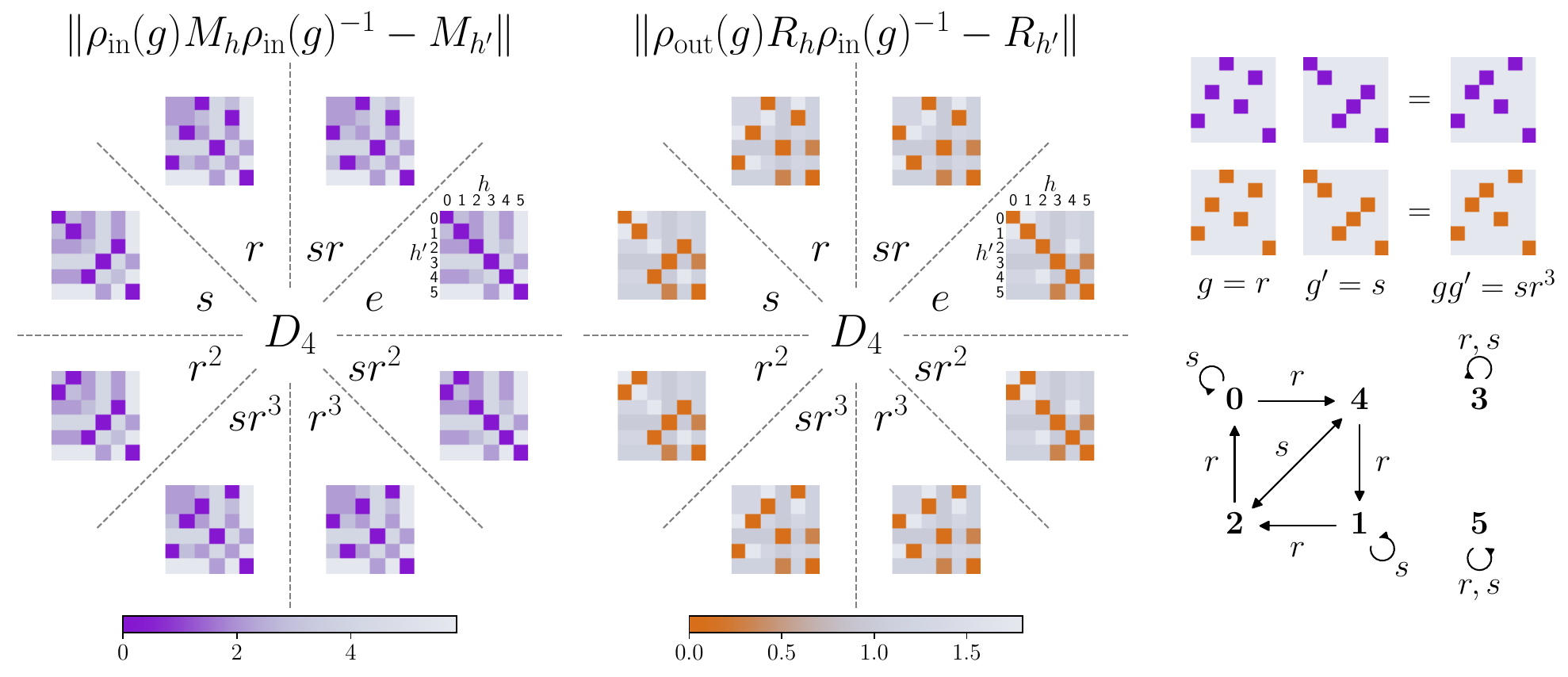}{0/1/b, .36/1/c, .705/1/d, .705/.52/e}
	\end{minipage}\hfill
	\begin{minipage}[c]{.15\textwidth}
		\centering
		\labeledgraphicabove{f}{\linewidth}{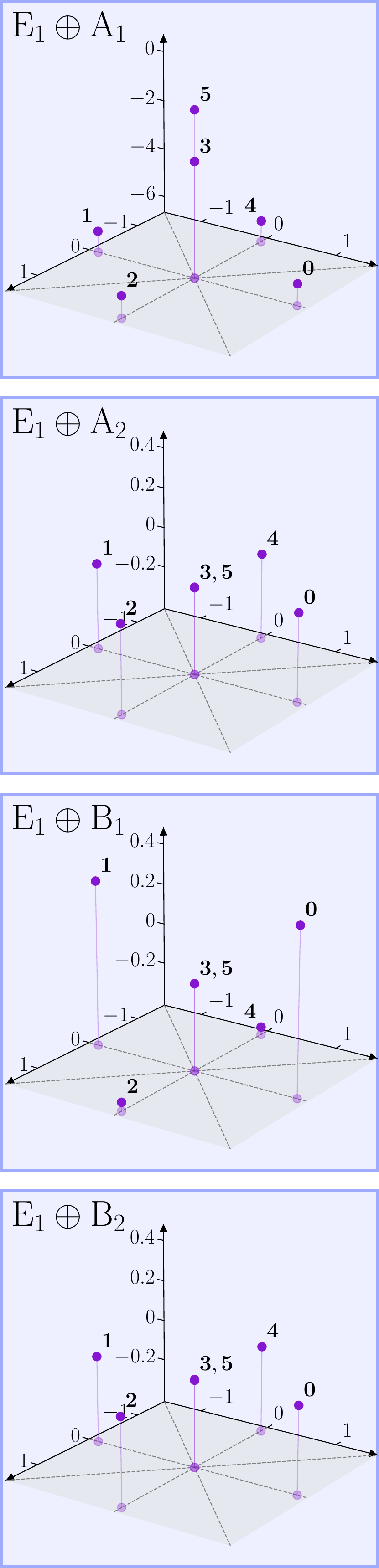}
	\end{minipage}
	\caption{%
		\textbf{(a)} Almost equivariant vision transformer experiment. Data: a random vector field drawn as one arrow per $16\times 16$ patch, with
		six equivariant targets per patch.
		Model:
		a $D_4$-equivariant vision transformer backbone followed by an unconstrained attention layer.
		\textbf{(b)} Distances $\lVert \rhoin(g) M_h \rhoin(g)^{-1} - M_{h'}\rVert$ between the QK matrices and their conjugates,
		one heatmap per $g\in D_4$, placed in the chamber of $g$.
		Theorem~\ref{thm:equiv-attn-char} predicts that
		the map $g\mapsto (h\mapsto\operatorname{argmin}_{h'}\lVert \rhoin(g) M_h \rhoin(g)^{-1} - M_{h'}\rVert)$
		is a homomorphism, and it holds empirically here.
		\textbf{(c)} The same for the OV matrices $R_h$.
		\textbf{(d)} Products of the resulting permutation matrices, where purple/orange $=1$ and gray $=0$.
		\textbf{(e)} The inferred action of $D_4$ on the heads: heads $0,1,2,4$ form one orbit, heads $3$ and $5$ form singleton orbits.
		\textbf{(f)} The six QK matrices $M_h$ projected onto four $D_4$-stable three-dimensional subspaces of
		$\mathbb{R}^{288\times 288}$. We pick the subspaces where the $M_h$ have most overlap.
	}
	\label{fig:vf}
	\sublabel{a}{fig:vector-fields}
	\sublabel{b}{fig:vf-conj-qk}
	\sublabel{c}{fig:vf-conj-ov}
	\sublabel{d}{fig:vf-perm}
	\sublabel{e}{fig:vf-orbits}
	\sublabel{f}{fig:vf-heads}
\end{figure}


We train a vision transformer to process an
input vector field represented by \emph{raster arrows} (see Fig.~\ref{fig:vector-fields}).
The model's task is to output the vector field $\mathbf{v}$ itself
and all four components of the derivative $\nabla \mathbf{v}$
(see Appendix~\ref{app:exp-vf} for details). In total, the target has $\Cout = 6$ components.
In this experiment, we do \emph{not} use an explicit equivariance loss.

We use a $D_4$-equivariant vision transformer \citep{nordstromQuickViTsSpeeding2026,ongUnifiedFrameworkVision2026}
with $\Cin=288$ channels ($36$ copies of the regular representation) and six heads,
with one twist: the final block is one unconstrained MHSA layer, directly followed by a tokenwise linear layer with $\Cout$ output channels (Fig.~\ref{fig:vector-fields}). That is, the last block does not have an MLP, residual connections, or any normalization. We also do not perform any invariantization, since the task is only equivariant and not invariant by design.

As before, the trained model found the structure predicted by Theorem~\ref{thm:equiv-attn-char}. Specifically,
we compute $\lVert \rhoin(g) M_h \rhoin(g)^{-1} - M_{h'}\rVert$ for each pair of heads $(h, h')$ and group element $g\in G$ (and similarly for the OV matrices) and
find that heads $0, 1, 2, 4$ organized themselves into
a $D_4$-orbit, and heads $3,5$ are each invariant (Fig.~\ref{fig:vf-conj-qk}--\ref{fig:vf-orbits}).


To visualize the geometric arrangement of the learned QK matrices $M_h$,
for each irrep $\rho$ of $D_4$, we pick the subspace of $\mathbb{R}^{288\times 288}$ isomorphic to $\rho$ that has the largest overlap with the $M_h$,
and project them onto the three-dimensional subspace spanned by the chosen $\mathrm{E}_1$ plane and a 1D irrep.
The results are shown in Fig.~\ref{fig:vf-heads}. In
$\mathrm{E}_1\oplus \mathrm{A}_2$ for example, the reflection generator $s$ flips the sign of the $\mathrm{A}_2$
component. But head $0$ is invariant under $s$, so its QK matrix $M_0$ is pinned at zero in the $\mathrm{A}_2$ axis.
The rotation generator then makes this into a flat square. The same happens in $\mathrm{E}_1\oplus \mathrm{B}_2$.








\section{Summary, Limitations \& Future Work}

We proved an exact structure theorem for equivariant MHSA (Theorem~\ref{thm:equiv-attn-char}),
which was then used to provide a lower bound on the number of equivariant models required to match the equivariant expressivity of a single unconstrained MHSA (Theorem~\ref{thm:components}). 
Finally, we illustrated in two experiments that the intricate structure predicted by Theorem~\ref{thm:equiv-attn-char} can be found ``in the wild'' in networks trained on equivariant tasks.

\textit{Other trade-offs}. The equivariance-expressivity trade-off, formalized by 
Theorem~\ref{thm:components} as well as the group-averaging argument 
in Sec.~\ref{sec:intro}, is in terms of parameter count and sets of expressible equivariant functions.
Other trade-offs, such as compute, are not accounted for. For example, achieving equivariance 
at the same parameter count without losing expressivity with $|G|$ times more heads
still results in $|G|$ times more operations in each forward, on top of 
the conjugation operations needed for weight sharing.
In the opposite direction, if one is willing to compromise expressivity, 
equivariance constraints at fixed width result in fewer operations in the Fourier domain due to Schur's lemma.

\textit{Going deeper.} Theorem~\ref{thm:equiv-attn-char} covers a single layer of MHSA. The same characterization
question remains open for an entire transformer, or even just for a single transformer
block. As a piece of empirical evidence, \citet{nandaProgressMeasuresGrokking2023} found that one transformer
block learns Fourier/irrep structures when trained on a modular addition task, which
is equivariant, suggesting a plausible extension of Theorem~\ref{thm:equiv-attn-char}.

\textit{Approximate equivariance.} The algebro-geometric approach in the proof of Theorem~\ref{thm:components}
worked because exact equivariance is translated to clean algebraic constraints
on the MHSA parameter space via Theorem~\ref{thm:equiv-attn-char}. In particular, Theorem~\ref{thm:equiv-attn-char} says nothing
about approximately equivariant architectures, or about approximating equivariant maps. 
It would be nice to have a statement of the form
$T$ approximately equivariant $\Rightarrow$ $(M, R)$ close to the equivariance locus,
in the form of bounds.
This could explain the pseudo-orbits we see in the modular addition experiments.

\newpage
\subsection*{AI use statement}
We used LLMs (Anthropic's Claude models, accessed through the Claude desktop app, Claude Code and the Claude VS Code extension, as well as Open AI's Codex via its VS Code extension from June-September 2026) for:
\begin{itemize}
\item Generating synthetic data and code for the experiments in Section.~\ref{sec:exps}.
\begin{itemize}
\item Source code was read and the synthetic data were inspected by one of the authors.
\item LLM-generated unit tests passed and were inspected by one of the authors.
\end{itemize}
\item Checking for typos.
\item Identifying proof gaps
\begin{itemize}
\item Every gap identified by the LLM was closed by one of the authors with their own human-written proof.
\end{itemize}
\item Checking the proofs of Theorem~\ref{thm:equiv-attn-char} and Theorem~\ref{thm:components} by having the LLM translate the authors' natural-language statements into Lean formalizations. 
The resulting Lean code compiled and one of the authors checked that the Lean formalizations of the statements correctly represented the natural-language statements.
\item Identifying relevant literature (checked by one of the authors).
\item Generating the TikZ and Matplotlib code for the figures.
\end{itemize}

We did not use LLMs for:
\begin{itemize}
\item Helping develop theoretical models or conceptual frameworks
\item Formulating any mathematical claim.
\item Suggesting presentation structure or style.
\item Designing or providing feedback on research methodology or experiments.
\item Proposing or refining hypotheses.
\item Assisting in natural-language translation.
\item Interpreting results.
\item Generating proofs.
\item Drafting any part of the paper or writing anything in general apart from correcting isolated typos.
\item Formatting references.
\end{itemize}

The following tasks are not applicable to this work:
\begin{itemize}
\item Supporting qualitative and thematic data analysis.
\item Cleaning and reformatting dataset
\end{itemize}

We have reviewed all LLM-assisted work. We take responsibility for the final content of this work, including text, claims or artifacts produced with the aid of generative AI.

\newpage
\bibliography{Refs, Refs_extra}
\bibliographystyle{iclr2027_conference}

\newpage
\appendix

\section{Related Work}

\textit{Equivariant architectures.}
Most existing equivariant neural network architectures are equivariant
convolutional neural networks
\citep{cohenGroupEquivariantConvolutional2016,kondorGeneralizationEquivarianceConvolution2018,bekkersRotoTranslationCovariantConvolutional2018,weiler_general_2019,weiler3DSteerableCNNs2018,cohenGeneralTheoryEquivariant2020,finziGeneralizingConvolutionalNeural2020,
	weilerCoordinateIndependentConvolutional2021,weiler_equivariant_2026}.
This paper says nothing about these.
Equivariant transformers, on the other hand, rely on equivariant attention,
which is the topic of this paper
and has been achieved either by lifting features to functions on the group
\citep{hutchinsonLieTransformerEquivariantSelfattention2021,romeroGroupEquivariantStandAlone2021,xu$E2$EquivariantVisionTransformer2023,fuVanillaGroupEquivariant2026}
or working directly in the Fourier space
\citep{fuchsSE3Transformers3DRotoTranslation2020,bokmanFloppingFLOPsLeveraging2025,nordstromQuickViTsSpeeding2026,
	ongUnifiedFrameworkVision2026}.

\textit{Universality.}
Universality of equivariant networks is usually studied in the approximate
sense, i.e., uniform convergence on compact sets, for
shallow or deep MLPs
\citep{yarotskyUniversalApproximationsInvariant2018,ravanbakhshUniversalEquivariantMultilayer2020,paciniUniversalityDeepEquivariant2025}.
Our work concerns the structure of the space of equivariant functions expressible by MHSA.
We do not prove anything about how good the different found components are in the sense of, say, approximating different classes of equivariant functions that are not expressible by a given MHSA.

\textit{Identifiability and neuroalgebraic geometry.}
Identifiability refers to the ability to recover the parameters of a
model, up to certain predefined ``gauge symmetries'', from the function it represents.
\cite{shahverdiIdentifiableEquivariantNetworks2026} observe
that identifiability together with the so-called \textit{adjunction property}
implies layerwise equivariance for equivariant functions.
\citet{henryGeometryLightningSelfAttention2026}
establish identifiability for lightning (``linear'', without softmax)
self-attention. More recently, \citet{henry_generic_2026}
characterizes the generic fiber of ordinary MHSA. The problem of understanding
the image and fibers of the realization map of a neural network can sometimes
be studied using tools from algebraic geometry. This is the neuroalgebraic
geometry program \citep{marchettiAlgebraUnveilsDeep2025}. For example,
\citet{kohnGeometryLinearNeural2025} study the equivariance locus of a
two-layer linear network for cyclic and permutation groups.
The irreducible components they find have the same origin as the ones for MHSA
described in Section~\ref{app:irrcomp-proof}.

\textit{Interpretability.}
Interpretability (or mechanistic interpretability) concerns trying to reverse engineer the underlying algorithms implemented by a neural network trained to solve a given task, as well as understanding how these algorithms emerge during training~\citep{cammarata2020circuits, powerGrokkingGeneralizationOverfitting2022}.
Often, simple model tasks are chosen in the literature, such as modular addition or multiplication in a small group~\cite{nandaProgressMeasuresGrokking2023,zhong2023the,wei2026deep}.
These tasks, and real world tasks alike, often contain symmetries.
Interpretability of how the symmetries are encoded in any fully equivariant MHSA is given by our Theorem~\ref{thm:equiv-attn-char}.

\section{Proof of Theorem \ref{thm:equiv-attn-char}}
\label{app:proof}

\subsection{Notational conventions in the appendix}
\label{app:einleitung}

In Appendices~\ref{app:proof}, \ref{app:irrcomp-proof}, \ref{app:equiv-bundle}, 
we will use slightly different notation from the main text:
\begin{itemize}
\item We use $\braket{x,y}$ instead of $x^\top y$ to denote the inner product of $x$ and $y$.
\item The feature spaces are now $V_{\text{in}}, V_{\text{out}}$ (or $V$) instead of $\mathbb{R}^{\Cin}$ or $\mathbb{R}^{\Cout}$.
\item We write $\mathrm{Hom}(V_{\mathrm{in}}, V_{\mathrm{out}})$ instead of $\mathbb{R}^{\Cout\times \Cin}$, and $\mathrm{End}(V_{\mathrm{in}})$ instead of $\mathbb{R}^{\Cin \times \Cin}$.
\item The adjoint of a lienar map with respect to an inner product is $\phi^*$ instead of $\phi^\top$.
\item $H$ is now the set of heads rather than the number of heads.
\item We will simply write $gx$ instead of $\rhoin(g)x$ or $\rhoout(g)x$.
\end{itemize}

We also introduce the notion of reduced MHSA, which is just attention parameterized directly by
the QK and OV matrices instead of query, key, value, and output projection maps.

\begin{definition}
For $(M_h)_{h=1}^H, (R_h)_{h=1}^H$ with
	$M_h\in \mathrm{End}(V_{\mathrm{in}})$, $R_h \in \mathrm{Hom}(V_{\mathrm{in}}, V_{\mathrm{out}})$, 
    we define \textbf{reduced multi-head self-attention (RMHSA)} to be
	\begin{equation}
    \label{eq:rmhsa}
		\mathrm{RMHSA}(x; (M_h)_{h\in H}, (R_h)_{h\in H})_i
		:= \sum_{h=1}^H 
        \frac{\sum_{j=1}^L\exp(\braket{x_i, M_h x_j})R_h x_j}{\sum_{j=1}^L \exp(\braket{x_i, M_hx_j)}}.
	\end{equation}
\end{definition}
Note that ``reduced'' refers to the reduction in parameterization redundancy,
not in expressive power. Indeed, a priori, RMHSA could have more expressive
power than MHSA if the ranks of the $M_h$ and $R_h$ are unconstrained.

\subsection{The proof}
We will show (3) $\Rightarrow$ (2) $\Rightarrow$ (1) $\Rightarrow$ (3). 
The implication (2) $\Rightarrow$ (1) is trivial.
(3) $\Rightarrow$ (2) is direct computation:
\begin{equation}
	\begin{aligned}
		 & \sum_{h\in H}
		\frac{\sum_{j=1}^L \exp(\braket{gx_i, M_h gx_j})R_h gx_j}{\sum_{j=1}^L\exp(\braket{gx_i, M_h gx_j})}
		=
		\sum_{p\in X}
		\frac{\sum_{j=1}^L \exp(\braket{gx_i, M_p gx_j})R_p gx_j}{\sum_{j=1}^L\exp(\braket{gx_i, M_p gx_j})} \\
		 & = \sum_{p\in X}
		\frac{\sum_{j=1}^L \exp(\braket{x_i, M_{g^{-1}p} x_j})gR_{g^{-1}p}x_j}{\sum_{j=1}^L\exp(\braket{x_i, M_{g^{-1}p} x_j})}
		= g\sum_{p\in X}
		\frac{\sum_{j=1}^L \exp(\braket{x_i, M_{p} x_j})R_{p}x_j}{\sum_{j=1}^L\exp(\braket{x_i, M_{p} x_j})} \\
		 & = g\sum_{h\in H}
		\frac{\sum_{j=1}^L \exp(\braket{x_i, M_h x_j})R_h x_j}{\sum_{j=1}^L\exp(\braket{x_i, M_h x_j})}.
	\end{aligned}
\end{equation}

We now prove (1) $\Rightarrow$ (3).

The equivariance condition for $L$ tokens $(x_1, \dotsc, x_L)\in \mathbb{R}^L \otimes V$
and $|H|$ heads $V=\bigoplus_{h\in H}V_h$ at the $i$-th output token reads
\begin{equation}
	\sum_{h\in H} \frac{\sum_{j=1}^L\exp(\braket{gx_i, M_h gx_j})R_h gx_j}{\sum_{j=1}^L\exp(\braket{gx_i, M_h gx_j})}
	= g
	\sum_{h\in H} \frac{\sum_{j=1}^L\exp(\braket{x_i, M_h x_j})R_hx_j}{\sum_{j=1}^L\exp(\braket{x_i, M_h x_j})}
\end{equation}
for all $g\in G$.
Hence,
\begin{equation}
	\label{eq:equiv-attn-b}
	\begin{aligned}
		 & \sum_{h\in H} \frac{\sum_{j=1}^L\exp(\braket{x_i, M_h^g x_j})R_h^g x_j}
		{\sum_{j=1}^L\exp(\braket{x_i, M_h^gx_j})}
		=
		\sum_{h\in H} \frac{\sum_{j=1}^L\exp(\braket{x_i, M_h x_j})R_h x_j}{\sum_{j=1}^L\exp(\braket{x_i, M_h x_j})},
	\end{aligned}
\end{equation}
where $M_h^g := gM_hg^{-1}$ and $R_h^g := gR_h g^{-1}$. Setting all tokens to
the same $x\in V$ yields $g\sum_h R_hg^{-1} x = \sum_h R_h x$, or
$\sum_h R_h \in \mathrm{End}_G(V)$.
Now, set $x_{j\ge 2} = x$, $x_1 = x+y$, and $i=2$ in Eq.~\eqref{eq:equiv-attn-b} to get
\begin{equation}
	\sum_{h\in H} \frac{(L-1) R_h^g x + e^{\braket{x, M_h^g y}}  R_h^g(x+y)}{L-1 + e^{\braket{x, M_h^g y}}}
	= \sum_{h\in H} \frac{(L-1) R_h x + e^{\braket{x, M_h y}}  R_h(x+y)}{L-1 + e^{\braket{x, M_h y}}}.
\end{equation}
Cancelling out $\sum_h R_h^g x = \sum_h R_h x$ yields
\begin{equation}
	\sum_{h\in H} \frac{e^{\braket{x, M_h^g y}}  R_h^gy}{L-1 + e^{\braket{x, M_h^g y}}}
	= \sum_{h\in H} \frac{e^{\braket{x, M_h y}}  R_hy}{L-1 + e^{\braket{x, M_h y}}}.
\end{equation}
Or, writing $\sigma_\alpha(s):= e^s / (\alpha + e^s)$,
\begin{equation}
	\label{eq:equiv-sigmoid}
	\sum_{h\in H} \sigma_{L-1}(\braket{x, M_h^g y})R_h^gy
	= \sum_{h\in H} \sigma_{L-1}(\braket{x, M_hy})R_hy.
\end{equation}
Differentiating Eq.~\eqref{eq:equiv-sigmoid} $n$ times with respect to $x$ at $x=0$ yields
\begin{equation}
	\sigma_{L-1}^{(n)}(0)
	\sum_{h\in H} (M_h^g y)^{\otimes n} \otimes R_h^gy
	=  \sigma_{L-1}^{(n)}(0)\sum_{h\in H}(M_hy)^{\otimes n}\otimes R_hy.
\end{equation}
Here, the superscript $(n)$ denotes the $n$-th derivative of a function.

\begin{lemma}
	\label{lem:sigmoid-derivatives}
	If $\alpha$ is an integer larger than one,
	then $\sigma_\alpha^{(n)}(0)\neq 0$ for all $n \in \mathbb{N}_{\ge 0}$.
	\begin{proof}
		Let $\lambda := 1/\alpha < 1$. We can rewrite the sigmoid-like function as
		\begin{equation}
			\sigma_\alpha(x)
			= \lambda e^x \frac{1}{1+\lambda e^x}
			= \lambda e^x \sum_{b=0}^\infty (-1)^b \lambda^b e^{bx}
			= \sum_{b=1}^\infty (-1)^{b+1} \lambda^b e^{bx}.
		\end{equation}
		Thus
		\begin{equation}
			\sigma_\alpha^{(n)}(0)
			= \sum_{b=1}^\infty (-1)^{b+1} \lambda^b b^n,
		\end{equation}
		which can be related to the Eulerian polynomials $A_n$ for $n \ge 1$ (the claim is clearly true for $n=0$):
		\begin{equation}
			\sigma^{(n)}_\alpha(0)
			= \frac{\lambda A_n(-\lambda)}{(1+\lambda)^{n+1}}.
		\end{equation}
		So $\sigma_{\alpha}^{(n)}(0) = 0\Leftrightarrow A_n(-\lambda)=0$.
		But $A_n$ is a monic polynomial with constant term $1$,
		so by the rational root theorem, the only possible rational roots are $\pm 1$.
	\end{proof}
\end{lemma}

Thus, if the number of tokens is at least three, then
for all $n \ge 0$ we have
\begin{equation}
	\sum_{h\in H}(M_h^gy)^{\otimes n} \otimes R_h^g y
	=  \sum_{h\in H}(M_hy)^{\otimes n} \otimes R_hy,
\end{equation}
or
\begin{equation}
	\label{eq:attn-matrix-equivariance-tensor}
	\sum_{h\in H}(M_hy)^{\otimes n} \otimes R_hy
	+ \sum_{h\in H}(M_h^gy)^{\otimes n} \otimes (-R_h^g y)
	= 0.
\end{equation}

\begin{lemma}
	\label{lem:tensor-separation}
	Let $V$ be a finite dimensional real vector space and let
	$(R_i)_{i\in I}, (M_i)_{i\in I}$
	be finite indexed families in $\mathrm{End}(V)$.
	Suppose
	\begin{equation}
		\label{eq:tensor-product-condition}
		\forall n \ge 0, \forall y\in V:
		\sum_{i\in I} (M_iy)^{\otimes n}\otimes R_i y= 0.
	\end{equation}
	Partition $I = \bigsqcup_{p\in Z} C_p$
	such that $M_i = M_j \Leftrightarrow \exists p:i,j \in C_p$.
	Then $\sum_{i\in C_p}R_i = 0$ for all $p\in Z$.
	\begin{proof}
		Roughly speaking, the proof idea is to first prove the claim pointwise
		and then employ a density argument.

		Let
		\begin{equation}
			\begin{aligned}
				S := & \{
				y \in V:
				\forall i,j: M_i y = M_j y \Leftrightarrow M_i = M_j
				\}             \\
				     & \cap \{
				y\in V: \forall i: M_i \neq 0 \Rightarrow M_iy \neq 0
				\}
			\end{aligned}
		\end{equation}
		Then $S$ is nonempty and open.

		Fix $y\in S$.
		Find $\alpha \in V^*$ that separates the $M_iy$,
		i.e. $\alpha(M_iy) = \alpha(M_jy)\Leftrightarrow M_i y = M_jy$.
		Furthermore, we may assume $\alpha(M_iy)=0\Leftrightarrow M_iy=0$
		($\Leftrightarrow M_i=0$ because $y\in S$).
		Define $c_p := \alpha(M_iy)$ for any $i\in C_p$.
		Then, $\alpha^{\otimes n}\otimes \mathrm{Id}$ evaluated on
		Eq.~\eqref{eq:tensor-product-condition} reads
		\begin{equation}
			\forall n \ge 0: \sum_{p\in Z}c_p^n \sum_{i\in C_p}R_iy = 0.
		\end{equation}
		First, suppose $c_p\neq 0$ for all $p\in Z$. Then
		the $c_p$ are nonzero and pairwise distinct.
		Thus, the matrix with entries $\mathbf{V}_{np} = c_p^n$, $n=0,1, \dotsc,
			|Z|-1$ is invertible, so $\sum_{i\in C_p}R_i y=0$ for all $p\in
			Z$ and $y\in S$.

		If $c_{p_0} = 0$ for some $p_0\in Z$ ($\Leftrightarrow M_i=0$ for some
		$i$), then the same argument still implies $\sum_{i\in C_p}R_i y =0$ for
		$p\in Z \setminus \{p_0\}$ and $y\in S$, which in turn implies $\sum_{i\in
				C_{p_0}}R_i y=0$ since $\sum_{i\in I}R_iy = 0$ (Eq.~\eqref{eq:tensor-product-condition} at $n=0$).

		Hence, $\sum_{i\in C_p}R_i y = 0$ for all $p\in Z$ and $y\in S$.
		Since $S$ is nonempty and open, $\sum_{i\in C_p}R_i = 0$.
	\end{proof}
\end{lemma}

Applied to Eq.~\eqref{eq:attn-matrix-equivariance-tensor}, the summation index
set is $I = H\sqcup H = \bigsqcup_{p\in Z} C_p$,
partitioned by the values of $(M_h)_{h\in H}$ and $(M_h^g)_{h\in H}$.
Partition $H = \bigsqcup_{p\in Y} B_p$ by $(M_h)_{h\in H}$, and write $M_p :=
	M_h$ for any $h\in B_p$ and $R_p := \sum_{h\in B_p}R_h$.
Clearly, we have $M_h^g = M_{h'}^g\Leftrightarrow$
$h, h'$ are in some $B_p$. That is, the partition of $H\sqcup H$ must be coarser
than $\left(\bigsqcup_{p\in Y} B_p\right)\sqcup \left(\bigsqcup_{p\in Y}B_p\right)$.
Define
\begin{equation}
	X_0 := \left\{p\in Y: R_p = 0\right\}\qquad X := Y \setminus X_0.
\end{equation}

Define a map $\pi_g: X \rightarrow X$ as follows: Take $p\in X$. Then
$R_p \neq 0$ by construction. So, by Lemma~\ref{lem:tensor-separation},
there must be exactly one $p'\in Y$ such that $M_p = M_{p'}^g$ and $R_p = R_{p'}^g$.
But then $R_{p'}\neq 0$, so $p'$ must belong to $X$.
Put $\pi_g(p) = p'$. If $\pi_g(p_1) = \pi_g(p_2)$, then $M_{p_1} = M_{p_2}$, so $\pi_g$ is bijective.

Next, we show that $g \mapsto \pi_g$ is an antihomomorphism. Take $g_1, g_2\in G$ and $p\in X$.
Let $p' = \pi_{g_1}(p)$, $p'' = \pi_{g_2}\circ\pi_{g_1}(p) = \pi_{g_2}(p')$.
Then $M_p = M_{p'}^{g_1}$ and $M_{p'} = M_{p''}^{g_2}$, so
$M_p = M_{p''}^{g_1g_2}$, i.e., $p'' = \pi_{g_1g_2}(p)$.
That is, $\pi_{g_1g_2} = \pi_{g_2}\circ \pi_{g_1}$.

Let $G$ act on $X$ by $g\cdot p = \pi_g^{-1}(p)$. Then $M_{g\cdot p} =
	M_{\pi_g^{-1}(p)} = M_{\pi_g\circ \pi_g^{-1}(p)}^{g} = M_p^g$, and similarly
$R_{gp} = R_p^g$. So $p \mapsto M_p$ and $p \mapsto R_p$ are $G$-equivariant on
$X$.
This concludes the proof of 
Theorem~\ref{thm:equiv-attn-char},
where the partial $G$-clustering is given by the natural surjection $\chi: H\rightarrow Y$
giving the partition.
This last argument is illustrated in Fig.~\ref{fig:group-action-argument}.

\begin{figure}
	\centering
%
%

\begingroup

\def\lgaScale{.7}         
\def\lgaMathScale{1.2}     
\def\lgaTriBeltScale{1.32} 

\def\lgaCapsule#1#2#3#4#5{%
	\begin{scope}[shift={(#1,#2)}, rotate=#3]
		\draw[lga belt]
		(-#4,#5)
		arc[start angle=90,end angle=270,radius=#5]
		-- (#4,-#5)
		arc[start angle=-90,end angle=90,radius=#5]
		-- cycle;
	\end{scope}%
}

\def\lgaTriBelt#1#2{%
	\begin{scope}[shift={(#1,#2)}, scale=\lgaTriBeltScale]
		\draw[lga belt]
		(0.121,0.470)
		arc[start angle=30,end angle=150,radius=0.14]
		-- (-0.481,-0.150)
		arc[start angle=150,end angle=270,radius=0.14]
		-- (0.360,-0.360)
		arc[start angle=270,end angle=390,radius=0.14]
		-- cycle;
	\end{scope}%
}

\begin{tikzpicture}[
	scale=\lgaScale,
	every node/.style={transform shape},
	lga label/.style={scale=\lgaMathScale},
	lga head/.style={
			circle, draw=black, fill=black,
			minimum size=4.8pt, inner sep=0pt, line width=0.7pt
		},
	lga belt/.style={line width=0.85pt},
	lga action/.style={
	line width=0.7pt,
	-{Stealth[length=2.2mm, width=1.8mm]}
	}
	]

	\node[lga label] at (1.95, 1.35) {$X$ \; ($R_p \neq 0$)};
	\node[lga label] at (7.35, 1.35) {$X_0$ \; ($R_p = 0$)};

	\draw[dashed] (5.25, 0.95) -- (5.25, -3.9);

	\node[lga head] at (0, 0) {};
	\lgaCapsule{0}{0}{0}{0}{0.17}
	\node[lga head] at (1.8, 0.51)   {};
	\node[lga head] at (1.32, -0.27) {};
	\node[lga head] at (2.28, -0.27) {};
	\lgaTriBelt{1.8}{0.03}
	\node[lga head] at (3.42, 0) {};
	\node[lga head] at (4.38, 0) {};
	\lgaCapsule{3.9}{0}{0}{0.48}{0.17}
	\node[lga head] at (6.12, 0) {};
	\node[lga head] at (7.08, 0) {};
	\lgaCapsule{6.6}{0}{0}{0.48}{0.17}
	\node[lga head] at (8.1, 0) {};
	\lgaCapsule{8.1}{0}{0}{0}{0.17}

	\node[lga head] at (0, -2.9) {};
	\lgaCapsule{0}{-2.9}{0}{0}{0.17}
	\node[lga head] at (1.8, -2.39)  {};
	\node[lga head] at (1.32, -3.17) {};
	\node[lga head] at (2.28, -3.17) {};
	\lgaTriBelt{1.8}{-2.87}
	\node[lga head] at (3.42, -2.9) {};
	\node[lga head] at (4.38, -2.9) {};
	\lgaCapsule{3.9}{-2.9}{0}{0.48}{0.17}
	\node[lga head] at (6.12, -2.9) {};
	\node[lga head] at (7.08, -2.9) {};
	\lgaCapsule{6.6}{-2.9}{0}{0.48}{0.17}
	\node[lga head] at (8.1, -2.9) {};
	\lgaCapsule{8.1}{-2.9}{0}{0}{0.17}

	\draw[lga action] (1.8, -0.80) -- (0,   -1.90);  
	\draw[lga action] (3.9, -0.80) -- (1.8, -1.90);  
	\draw[lga action] (0,   -0.80) -- (3.9, -1.90);  
	\node[lga label] at (-0.75, -1.35) {$\pi_g$};


\end{tikzpicture}

\endgroup
	\caption{Illustration of the argument that applies Lemma~\ref{lem:tensor-separation} to
		Eq.~\eqref{eq:attn-matrix-equivariance-tensor} to get a partial $G$-clustering.
	}
	\label{fig:group-action-argument}
\end{figure}

Finally, we would like to remark that the same technique can be used to prove
the following identifiability result, which slightly generalizes
\cite{henry_generic_2026}:
\begin{theorem}
	\label{thm:identifiability}
	Suppose $M,R\in \mathrm{End}(V)^{H}$ and $\bar M, \bar R\in \mathrm{End}(V)^{\bar H}$ satisfy
	\begin{equation}
		\mathrm{RMHSA}(-; M, R) = \mathrm{RMHSA}(-;\bar M,\bar R):
		\mathbb{R}^L\otimes V \rightarrow \mathbb{R}^L\otimes V
	\end{equation}
	for some $L \ge 3$.
	Then
	\begin{equation}
		\sum_{h: M_h= m} R_h = \sum_{\bar h: \bar M_{\bar h} = m} \bar R_{\bar h}
	\end{equation}
	for any $m\in \mathrm{End}(V)$.
\end{theorem}

\section{Geometry of the Equivariance Locus}
\label{app:irrcomp-proof}

The proof is split into four sections.
Section~\ref{sec:equivariant-bundle-motivation} merely motivates the machinery
we use; the actual proof is completely contained in
Sections~\ref{sec:proof-notations}, \ref{sec:step1-equivariance-locus}, and \ref{sec:step2-function-space}.

\subsection{Motivation: Equivariant Vector Bundles}
\label{sec:equivariant-bundle-motivation}

In the proof, we will make heavy use of \textit{equivariant vector bundles},
the definition and properties of which are reviewed in
Appendix~\ref{app:equiv-bundle}.
This is but a convenient way to pack a lot of
information into one single object, motivated in three steps as follows:
\begin{enumerate}[itemsep=4pt, topsep=5pt, parsep=2pt]
	\item[(1)] Let $V, W$ be $G$-representations. Suppose we want to study the family $\mathcal{F}$ of $G$-equivariant linear maps
	      $M\in \mathrm{Hom}_G(V, W)$ with $\mathrm{rk}(M) \le r$.
	      It is true that any such $M$ factorizes through
	      some $G$-representation $U$ of dimension $\le r$. That is,
	      $M = \psi_1\psi_2$ with $\psi_2\in \mathrm{Hom}_G(V, U)$ and $\psi_1\in \mathrm{Hom}_G(U, W)$.
	      Thus, $\mathcal{F} = \bigcup_{U} \mathcal{D}(V, W; U)$, where
	      $\mathcal{D}(V, W;U):= \mathrm{Hom}_G(U, W)\mathrm{Hom}_G(V, U)$,
	      and $U$ ranges over all $G$-representations of dimension at most $r$ up to isomorphism.
	\item[(2)] Fix a tuple of integers $(r_h)_{h\in H}$, one for each head.
	      Suppose we now want to study the family $\mathcal{F}$ of tuples $M \in \mathrm{Hom}_G(V,
		      W)^H$ such that each $M_h$ has rank at most $r_h$. The same argument in (1)
	      applies headwise, so, for each $M \in \mathcal{F}$, we can find $(U_h)_{h\in H}$, where
	      each $U_h$ is a $G$-representation of dimension at most $r_h$, such
	      that $M_h \in \mathcal{D}(V, W; U_h)$.
	      Now, pack the $U_h$ into $E:= (U_h)_{h\in H}$, and define
	      $\Gamma(E):=\bigoplus_{h\in H}U_h$, which is a $G$-representation.
	      Let $\mathcal{D}(V, W; E)$ be the set of tuples $M \in
		      \mathrm{Hom}_G(V, W)^H$ such that there exists $\phi_2 \in
		      \mathrm{Hom}_G(V, \Gamma(E)), \phi_1\in \mathrm{Hom}_G(\Gamma(E), W)$
	      satisfying $M_h = \phi_1\mathrm{ev}_h \phi_2$, where $\mathrm{ev}_h:
		      \Gamma(E)\rightarrow \Gamma(E)$ projects onto the $h$-th summand.
	      Hence, we have $\mathcal{F} = \bigcup_{E}\mathcal{D}(V, W; E)$. Nothing
	      fancy is done: we simply repackaged the $U_h$ into one object, $E$,
	      which is an equivariant bundle with trivial action on the base space
	      $H$.
	      This could seem somewhat ad-hoc, since each fiber $U_h$ is really just an
	      independent $G$-representation. In particular $\mathcal{D}(V, W; E) =
		      \prod_{h\in H}\mathcal{D}(V, W; U_h)$, and $\mathcal{F}$ is also a product.
	\item[(3)]
	      Finally, let $H = G/K$ be a homogeneous space, and let $[e]$ ($= K$) denote the identity coset of $K$.
	      Suppose we want to study the family $\mathcal{F}$ of tuples $M\in \mathrm{Hom}(V, W)^H$
	      such that $h\mapsto M_h$ is $G$-equivariant and $\mathrm{rk}(M_h) \le r_h$ for each $h$.
	      That is, $M_{gh} = gM_hg^{-1}$, which is exactly the equivariance condition
	      in Theorem~\ref{thm:equiv-attn-char}.
	      The family $\mathcal{F}$ is not a product over $H$ anymore because $M_{[e]}$
	      determines all the other $M_{h}$ for $h\in G/K$.
	      However, (1) still applies to $M_{[e]}$ for the subgroup $K$:
	      because $M_{[e]} = M_{[ke]} = kM_{[e]}k^{-1}$, we have $M_{[e]} \in
		      \mathrm{Hom}_K(\mathrm{Res}^G_K(V),\mathrm{Res}^G_K(W))$,
	      so there is a $K$-representation of dimension at most $r_{[e]}$
	      such that
	      \begin{equation}
		      \label{eq:homog-basepoint-K-equivariant}
		      M_{[e]} \in
		      \mathrm{Hom}_K(U, \mathrm{Res}^G_K(W))\mathrm{Hom}_K(\mathrm{Res}^G_K(V),U).
	      \end{equation}

	      By Frobenius reciprocity,
	      \begin{equation}
		      \begin{aligned}
			       & \mathrm{Hom}_K(U, \mathrm{Res}_K^G(W))\cong \mathrm{Hom}_G(\mathrm{Ind}_K^G(U), W) \\
			       & \mathrm{Hom}_K(\mathrm{Res}^G_K(V),U) \cong \mathrm{Hom}_G(V, \mathrm{Ind}_K^G(U))
		      \end{aligned}
	      \end{equation}
	      canonically, so
	      any decomposition $M_{[e]} = \psi_1\psi_2$ with respect to
	      Eq.~\eqref{eq:homog-basepoint-K-equivariant}
	      gives two $G$-equivariant maps: $\phi_1\in\mathrm{Hom}_G(\mathrm{Ind}_K^G(U), W)$
	      and
	      $\phi_2\in\mathrm{Hom}_G(V, \mathrm{Ind}_K^G(U))$.
	      Now, $\mathrm{Ind}_K^G(U)$ is nothing but the space of sections $\Gamma(E)$ of the associated
	      bundle $E = G \times_K U$ (where we view $G \rightarrow G/K$ as a principal $K$-bundle),
	      and it is easily shown (see Appendix~\ref{app:equiv-bundle}) that
	      the entire tuple $(M_h)_h$ is given by
	      $M_h = \phi_1 \mathrm{ev}_h \phi_2$, where $\mathrm{ev}_h: \Gamma(E)\rightarrow \Gamma(E)$
	      again projects onto the $h$-th fiber.
	      The upshot is that $\mathcal{F} = \bigcup_E \mathcal{D}(V, W;E)$ still holds
	      with the same definition of $\mathcal{D}(V, W; E)$,
	      but now the union is taken over all
	      $G$-equivariant vector bundles over $H = G/K$, which are in bijection with $K$-representations
	      via $U \mapsto G\times_K U$.
	      In this case, the rank constraint is equivalent to $\mathrm{rk}_h(E) := \dim
		      (E_h) = \dim U \le r_h$.

\end{enumerate}

In the context of Theorem~\ref{thm:equiv-attn-char}, $V=W$ and we are given a
partial $G$-clustering $\chi: H \rightarrow X \sqcup X_0$, and the objective is
to study the family of tuples $(M_p)_{p\in X} \in \mathrm{End}(V)^X$ satisfying
$M_{gp} = gM_pg^{-1}$ and $\mathrm{rank}(M_p) \le \dim V_h$ for all $h\in
	\chi^{-1}(p)$ (and similarly for the $R_p$ with slightly different rank
constraints). If $\chi = \mathrm{Id}$ (which means a trivial action on $H$ with
$X_0=\varnothing$, see Example~\ref{ex:trivial-clustering}), then we are in
case (2). In the general case, decompose $X$ into $G$-orbits, and apply (3) to
each orbit $O\subset X$ to get an equivariant bundle $E_O$, which can then be
assembled into an equivariant bundle $E$ over $X$, where the local ranks can
depend on the orbit.

\subsection{Step 0: Setup and Notations}
\label{sec:proof-notations}

We will start by fixing some notations.
Recall the setting: the token feature
space $V$ is a finite-dimensional orthogonal $G$-representation with some given
orthogonal decomposition $V = \bigoplus_{h\in H} V_h$. Assume $L\ge 3$.

For a $G$-equivariant vector bundle $E\rightarrow X$, define
\begin{equation}
	\label{eq:bundle-D-def}
	\begin{aligned}
		\mathcal{D}(V; E):= \Big\{
		(\phi_1\mathrm{ev}_p\phi_2)_{p\in X}:
		\phi_1\in \mathrm{Hom}_G(\Gamma(E), V),
		\phi_2\in \mathrm{Hom}_G(V, \Gamma(E))
		\Big\}
		\subset \mathrm{End}(V)^X.
	\end{aligned}
\end{equation}
This is the set of all elements $M$ in $\mathrm{End}(V)^X$
that satisfy the equivariance condition $M_{gp} = gM_pg^{-1}$
and factor through the bundle $E$ in the sense of (3) of Section~\ref{sec:equivariant-bundle-motivation}.

For $r: X\rightarrow \mathbb{N}_{\ge 0}$, define
\begin{equation}
	\label{eq:gset-Dr-def}
	\mathcal{D}_r(V; X) := \Big\{
	M\in \mathrm{End}(V)^X:
	\mathrm{rank}(M_p) \le r(p),
	M_{gp} = gM_pg^{-1}
	\Big\},
\end{equation}
which is the family of elements in $\mathrm{End}(V)^X$ that satisfy the same
equivariance condition and are rank constrained by $r$, i.e., what we called
$\mathcal{F}$ in Section~\ref{sec:equivariant-bundle-motivation}.

Let $\mathrm{Sub}^r(E)$ be the set of equivalence classes of maximal
$r$-rank-bounded subbundles of $E$ (Definition~\ref{def:equiv-bundle}) up to bundle isomorphisms, where maximality
is defined with respect to inclusions. Then, the argument in Section~\ref{sec:equivariant-bundle-motivation}
is formalized as
\begin{equation}
	\label{eq:equiv-system-rank-bounded}
	\mathcal{D}_r(V; X)= \bigcup_{E \in \mathrm{Sub}^r(X\times V)} \mathcal{D}(V; E).
\end{equation}
See Lemma~\ref{lem:equiv-system-rank-bounded-existence} in
Appendix~\ref{app:equiv-bundle} for a proof.

Define
\begin{equation}
	\label{eq:ZE-def}
	\begin{aligned}
		\mathcal{Z} & := \{(\phi_2\pi_h\phi_1)_{h\in H}: \phi_1,\phi_2\in\mathrm{End}(V)\}
		= \{B\in \mathrm{End}(V)^H: \mathrm{rank}(B_h)\le\dim V_h\}                                                                                       \\
		            & \subset \mathrm{End}(V)^H                                                                                                           \\
		\mathcal{E} & := \Big\{ (M, R)\in \mathcal{Z}\times \mathcal{Z}: \mathrm{RMHSA}(-; M, R): \mathbb{R}^L\otimes V \rightarrow \mathbb{R}^L\otimes V
		\text{ is } G\text{-equivariant}
		\Big\}                                                                                                                                            \\
		            & \subset \mathrm{End}(V)^{H}\oplus \mathrm{End}(V)^H,
	\end{aligned}
\end{equation}
That is, $\mathcal{E}$ is the subset of the RMHSA equivariance locus
satisfying the rank constraints defining $\mathcal{Z}$.

For each partial $G$-clustering $\chi: H\rightarrow X_0\sqcup X$, define the
\textit{attention width} $a^\chi$ and \textit{value width} $u^\chi$ as

\begin{equation}
	\label{eq:attn-val-width-def}
	\begin{aligned}
		a^\chi & : X \rightarrow \mathbb{N}_{\ge 0}
		       & \qquad\qquad\qquad u^\chi                & : X \rightarrow \mathbb{N}_{\ge 0}         \\
		p      & \mapsto \min_{h\in \chi^{-1}(p)}\dim V_h
		       & \qquad\qquad\qquad p                     & \mapsto \sum_{h\in \chi^{-1}(p)} \dim V_h.
	\end{aligned}
\end{equation}

Now, set
\begin{equation}
	\mathcal{E}^\chi :=
	\Big[\mathcal{P}_\chi^*\Big(\mathrm{End}(V)^{X_0}\times \mathcal{D}_{a^\chi}(V; X)\Big)\cap \mathcal{Z}
		\Big]
	\times
	\Big[\mathcal{P}_\chi^{-1}\Big(\{0\}^{X_0}\times \mathcal{D}_{u^\chi}(V; X)\Big)\cap \mathcal{Z}
		\Big],
\end{equation}
where $\mathcal{D}_{a^\chi}(V; X)$ and $\mathcal{D}_{u^\chi}(V; X)$ are defined
in Eq.~\eqref{eq:gset-Dr-def}.
The set $\mathcal{E}^\chi$ is the subset of $\mathrm{End}(V)^H\oplus \mathrm{End}(V)^H$
consisting of the RMHSA parameters corresponding to the partial $G$-clustering $\chi$,
rank-constrained by the attention and value widths.
Clearly, $\mathcal{E}^\chi =
	\mathcal{E}^{\chi'}$ if $\chi \cong \chi'$.
By Theorem~\ref{thm:equiv-attn-char-2}, we have $\mathcal{E} = \bigcup_\chi
	\mathcal{E}^\chi$, where the union is taken over all partial $G$-clusterings
$\chi: H \rightarrow X_0\sqcup X$ up to isomorphism.

For $E\in \mathrm{Sub}^{a^\chi}(X\times V)$ and $E'\in
	\mathrm{Sub}^{u^\chi}(X\times V)$, define
\begin{equation}
	\label{eq:EchiEE-def}
	\mathcal{E}^{\chi, E, E'}
	= \Big[\mathcal{P}_\chi^*\Big(\mathrm{End}(V)^{X_0}\times \mathcal{D}(V; E)\Big)\cap \mathcal{Z}\Big] \times
	\Big[\mathcal{P}_\chi^{-1}\Big(\{0\}^{X_0}\times \mathcal{D}(V;E')\Big)\cap \mathcal{Z}\Big].
\end{equation}

\begin{proposition}
	$\mathcal{E}^\chi = \bigcup_{E\in \mathrm{Sub}^{a^\chi}(X\times
			V)}\bigcup_{E'\in \mathrm{Sub}^{u^\chi}(X\times V)} \mathcal{E}^{\chi, E,
			E'}$.
	\begin{proof}
		This is a straightforward consequence of Lemma~\ref{lem:equiv-system-rank-bounded-existence},
		more precisely of Eq.~\eqref{eq:equiv-system-rank-bounded}.
	\end{proof}
\end{proposition}

\subsection{Step 1: Geometry of the Equivariance Locus in the RMHSA Parameter Space}
\label{sec:step1-equivariance-locus}

In the following (especially in Section~\ref{sec:step2-function-space}), we will let $\mathfrak{R}: \mathrm{End}(V)^H\oplus
	\mathrm{End}(V)^H \rightarrow \{\mathbb{R}^L\otimes V \rightarrow \mathbb{R}^L\otimes V\}$
denote the realization map of RMHSA.
Ultimately, we are interested in the class of functions $\mathfrak{R}(\mathcal{E})$.
However, we will first work in $\mathrm{End}(V)^H\oplus \mathrm{End}(V)^H$,
where algebraic geometry is more easily applicable.


\begin{proposition}
	\leavevmode
	\begin{enumerate}
		\item[(1)] $\mathcal{E}^{\chi, E, E'} \subset \mathrm{End}(V)^H\oplus \mathrm{End}(V)^H$
		      is Zariski-closed. If $\chi$ is injective (hence bijective), then $\mathcal{E}^{\chi, E, E'}$
		      is irreducible.
		\item[(2)] For a fixed $\chi$, if $(E_1, E'_1)$ and $(E_2, E'_2)$ are not isomorphic, then neither $\mathcal{E}^{\chi, E_1, E_1'}$
		      nor $\mathcal{E}^{\chi, E_2, E_2'}$ contains the other.
		\item[(3)]
		      Consider the trivial $G$-clustering $\mathrm{Id}: H \rightarrow H=X$ (Example~\ref{ex:trivial-clustering}).
		      For any $E\in \mathrm{Sub}^{a^\mathrm{Id}}(X\times V)$ and $E'\in \mathrm{Sub}^{u^\mathrm{Id}}(X\times V)$
		      without zero fibers (i.e. $\dim E_p, \dim E'_p > 0$ for all $p\in X=H$),
		      the set $\mathcal{E}^{\mathrm{Id}, E, E'}$ is an irreducible component
		      (i.e., maximal irreducible subset) of $\mathcal{E}$.

	\end{enumerate}
	\begin{proof}
		\;\\
		(1)\\
		Since $\mathcal{P}_\chi^*$ is an injective linear map, the set $\mathcal{E}^{\chi, E,
				E'}$, as defined by Eq.~\eqref{eq:EchiEE-def}, is closed by Lemma~\ref{lem:maps-through-bundle-closed}.

		If $\chi$ is injective, then without loss of generality $\chi = \mathrm{Id}_H$,
		in which case $\mathcal{P}_\chi = \mathrm{Id}_H$.
		So Eq.~\eqref{eq:EchiEE-def} reduces to
		\begin{equation}
			\mathcal{E}^{\chi, E, E'} =
			\Big[\big(\mathrm{End}(V)^{X_0} \times \mathcal{D}(V;E)\big)\cap \mathcal{Z}\Big]
			\times \Big[\big(\{0\}^{X_0} \times \mathcal{D}(V; E')\big)\cap \mathcal{Z}\Big].
		\end{equation}
		This is irreducible because it is a Cartesian product of irreducible sets.
		More precisely, $\mathcal{Z}\subset \mathrm{End}(V)^{X_0\sqcup X}$ (defined in
		Eq.~\eqref{eq:ZE-def}) is rectangular with respect to $\mathrm{End}(V)^{X_0}\times \mathrm{End}(V)^X$,
		and both $\mathcal{D}(V; E)$ and $\mathcal{D}(V; E')$ are irreducible by
		Lemma~\ref{lem:maps-through-bundle-closed}.

		(2)\\
		Suppose $\mathcal{E}^{\chi, E_1, E_1'}\subset \mathcal{E}^{\chi, E_2, E_2'}$.
		Then
		\begin{equation}
			\begin{aligned}
				\mathcal{P}_\chi^*\Big(\mathrm{End}(V)^{X_0}\times \mathcal{D}(V; E_1)\Big)\cap \mathcal{Z}
				\subset
				\mathcal{P}_\chi^*\Big(\mathrm{End}(V)^{X_0}\times \mathcal{D}(V; E_2)\Big)\cap \mathcal{Z}.
			\end{aligned}
		\end{equation}
		Apply $\mathcal{P}_\chi(\cdot)$ to both sides to get
		\footnote{This needs some care.
		$\mathcal{Z} = \mathcal{Z}_0 \times \mathcal{Z}_1 \subset \mathrm{End}(V)^{X_0}\times \mathrm{End}(V)^{X}$
		is rectangular, and the projection can be also split accordingly into $\mathcal{P}_0=\mathcal{P}_\chi|_{\chi^{-1}(X_0)}, \mathcal{P}_1=\mathcal{P}_{\chi}|_{\chi^{-1}(X)}$ (slight abuse of notation).
		Then use $\mathcal{D}(V;E_1)\cap \mathcal{Z}_1 = \mathcal{D}(V;E_1)$ and
		the injectivity of $\mathcal{P}_1 \mathcal{P}_1^*$.
		}
		$\mathcal{D}(V; E_1)
			\subset \mathcal{D}(V; E_2)$.
		Similarly, we have
		\begin{equation}
			\begin{aligned}
				\mathcal{P}_\chi^{-1}\Big(\{0\}^{X_0}\times \mathcal{D}(V; E_1')\Big)\cap \mathcal{Z}
				\subset
				\mathcal{P}_\chi^{-1}\Big(\{0\}^{X_0}\times \mathcal{D}(V; E_2')\Big)\cap \mathcal{Z}.
			\end{aligned}
		\end{equation}
		Again, apply $\mathcal{P}_\chi(\cdot)$ to get
		\footnote{Similarly,
		the left hand side is
		$(\dotsb)\times [\mathcal{P}_1^{-1}(\mathcal{D}(V; E_1'))\cap \mathcal{Z}_1]$,
		and $\mathcal{P}_1(\cdot)$ on the second factor recovers
		$\mathcal{D}(V;E_1')$.
		}
		$\mathcal{D}(V; E_1') \subset \mathcal{D}(V; E_2')$.

		By Lemma~\ref{lem:maps-through-bundle-closed}, $E_1$ is isomorphic to a
		subbundle of $E_2$ and $E_1'$ is isomorphic to a subbundle of $E_2'$.

		By maximality, we must have $E_1 \cong E_2$ and $E_1' \cong E_2'$.

		(3)\\
		By (1),
		$\mathcal{E}^{\mathrm{Id}, E, E'}$ is irreducible. It remains to show maximality.
		Suppose $\mathcal{E}^{\mathrm{Id}, E, E'} \subset \mathcal{E}^{\tilde \chi, \tilde E, \tilde E'}$
		for some partial $G$-clustering $\tilde \chi: H \rightarrow \tilde X_0\sqcup \tilde X$.
		If $\mathrm{Id}\cong \tilde \chi$, that is, $\tilde X_0=\varnothing$, $|\tilde X|=|H|$, and $G$ acts on $\tilde X$ trivially,
		then we must have $(E, E')\cong (\tilde E, \tilde E')$ by (2), assuming $\tilde \chi = \mathrm{Id}$ without loss of generality.
		Thus, suppose $\mathrm{Id}\not\cong \tilde\chi$.
		There are three (not mutually exclusive) possibilities:
		\begin{enumerate}
			\item[(a)] $\tilde X_0\neq \varnothing$
			\item[(b)] $\tilde \chi$ has a nontrivial fiber
			\item[(c)] $\tilde X\neq \varnothing$ and the $G$-action on $\tilde X$ is not trivial.
		\end{enumerate}
		In each of these three cases, we can show $\mathcal{E}^{\mathrm{Id}, E, E'}\setminus \mathcal{E}^{\tilde \chi, \tilde E,
				\tilde E'}$ is not empty:

		In case (a), we can find $(M,R)\in \mathcal{E}^{\mathrm{Id}, E, E'}$
		such that $\sum_{h\in \tilde\chi^{-1}(\tilde X_0)}R_h\neq 0$.\\
		In case (b), suppose $\tilde\chi(h) = \tilde \chi(h')$ with $h\neq h'$.
		Find $(M, R)\in \mathcal{E}^{\mathrm{Id}, E, E'}$ such that $M_h \neq M_{h'}$.\\
		In case (c), we may assume $\tilde X_0=\varnothing$ and $\tilde \chi$ is injective,
		and suppose $g\cdot p = p'$ with $p\neq p'\in \tilde X$.
		Then simply take $(M, R)\in \mathcal{E}^{\mathrm{Id}, E, E'}$
		such that $gM_{\tilde\chi^{-1}(p)}g^{-1}\neq M_{\tilde\chi^{-1}(p')}$.
	\end{proof}
\end{proposition}

\begin{corollary}
	\label{cor:polynomial-only-one-component}
	Let $\beta: \Lambda\rightarrow \mathrm{End}(V)^4$ be a polynomial map, where $\Lambda$
	is a finite-dimensional real vector space.
	Define
	\begin{equation}
		\mathcal{E}^{\beta}
		:= \Big\{
		((\phi_q^* \pi_h \phi_k)_{h\in H}, (\phi_o \pi_h\phi_v)_{h\in H}):
		(\phi_q, \phi_k, \phi_v, \phi_o) \in \beta(\Lambda)
		\Big\}\subset \mathrm{End}(V)^H \oplus \mathrm{End}(V)^H.
	\end{equation}
	If $\mathcal{E}^\beta\subset \mathcal{E}$, then there is a triple $(\chi, E, E')$
	such that $\mathcal{E}^{\beta}\subset \mathcal{E}^{\chi, E, E'}$.
	\begin{proof}
		$\mathcal{E}^{\beta}\subset \mathcal{E}$
		implies $\mathcal{E}^{\beta} = \bigcup_{(\chi, E, E')}\mathcal{E}^{\chi, E, E'} \cap \mathcal{E}^{\beta}$.
		Since the $\mathcal{E}^{\chi, E, E'}$ are Zariski-closed, each $\mathcal{E}^{\beta} \cap \mathcal{E}^{\chi, E, E'}$
		is closed in $\mathcal{E}^{\beta}$.
		But $\mathcal{E}^\beta$ is irreducible: it is the image of an irreducible set under a polynomial map.
		Thus, $\mathcal{E}^{\beta} \subset \mathcal{E}^{\chi, E, E'}$
		for some $(\chi, E, E')$.
	\end{proof}
\end{corollary}

\begin{corollary}
	\label{thm:param-space-irrcomp-estimate}
	The set $\mathcal{E}$ has at least
	$|\mathrm{Sub}^{(\dim V_h)_{h\in H}}_+(H\times V)|^2$
	irreducible components, where the subscript ``$+$'' means bundles with zero fibers are
	excluded.
\end{corollary}

Note that $\mathrm{Sub}_+^{(\dim V_h)_{h\in H}}(H\times V)\cong \prod_{h\in
		H}\mathrm{Sub}_+^{\dim V_h}(V)$,
where we think of $V$ as the $G$-equivariant vector bundle over one point.
Equivalently, $\mathrm{Sub}^r_+(V)$ consists of equivalence classes of maximal
$G$-subrepresentations of dimension at most $r$. In the main text, $|\mathrm{Sub}_+^r(V)|$
is denoted by $c_{G, r}(V)$.

\begin{example}
	\leavevmode
	\begin{enumerate}
		\item[(1)] Let $G= C_3$ act on $\mathbb{R}^2$ by $r\mapsto R(2\pi/3)$,
		      which is irreducible (and equivalent to $\mathrm{E}_1$).
		      Then $\mathrm{Sub}_+^1(\mathbb{R}^2) = \varnothing$
		      because there are no one-dimensional subrepresentations.
		\item[(2)] For the regular representation $\mathbb{R}[D_4]$ of the dihedral group $D_4$,
		      which decomposes as $\mathrm{A}_1\oplus \mathrm{A}_2\oplus \mathrm{B}_1\oplus \mathrm{B}_2\oplus 2\mathrm{E}_1$,
		      we have
		      \begin{equation}
			      \begin{aligned}
				       & \mathrm{Sub}^2(\mathbb{R}[D_4])
				      =\mathrm{Sub}^2_+(\mathbb{R}[D_4])                                                                        \\
				       & = \{\mathrm{A}_1\oplus \mathrm{A}_2, \mathrm{A}_1\oplus \mathrm{B}_1, \mathrm{A}_1\oplus \mathrm{B}_2,
				      \mathrm{A}_2\oplus \mathrm{B}_1, \mathrm{A}_2\oplus \mathrm{B}_2, \mathrm{B}_1\oplus \mathrm{B}_2, \mathrm{E}_1
				      \}.
			      \end{aligned}
		      \end{equation}
	\end{enumerate}
\end{example}

\subsection{Step 2: Passing to the Function Space}
\label{sec:step2-function-space}
Corollary~\ref{thm:param-space-irrcomp-estimate} is
a statement made in the parameter space of $\mathrm{RMHSA}$.
Transferring this to the function space requires modding out permutations:

\begin{proposition}
	\label{prop:containment-control}
	Suppose $\mathfrak{R}(\mathcal{E}^{\mathrm{Id}, E, E'}) \subset
		\bigcup_{k=1}^N \mathfrak{R}(\mathcal{E}^{\chi^k, \tilde E^k, \tilde E'^k})$,
	with
	$E, E'\in\mathrm{Sub}_+^{(\dim V_h)_{h\in H}}(H\times V)$
	and
	$\tilde E^k\in \mathrm{Sub}^{a^{\chi^k}}(X^k\times V), \tilde E'^k\in\mathrm{Sub}^{u^{\chi^k}}(X^k\times V)$.
	Then there is a $k\in \{1, \dotsc, N\}$
	and a permutation $\nu \in \mathrm{Aut}(H)$
	such that $\chi^k \cong \mathrm{Id}$
	and
	$E \hookrightarrow \tilde E^{k,\nu}$ and
	$E' \hookrightarrow \tilde E^{\prime k, \nu}$,
	where $\tilde E^{k,\nu}, \tilde E^{\prime k,\nu}$ denote the permuted bundles.

	If all heads have the same dimension $d:= \dim V / |H| = \dim V_h$,
	then the same statement holds with $E \cong \tilde E^{k, \nu}$
	and $E' \cong \tilde E^{\prime k, \nu}$.

	(Note: we have implicitly identified $\chi^k$ with $\mathrm{Id}$, which is harmless because
	there is only one isomorphism.)

	\begin{proof}
		Let
		\begin{equation}
			U := (\mathcal{D}(V; E)\setminus \mathcal{Q} )
			\times (\mathcal{D}(V; E') \setminus \mathcal{Q}'),
		\end{equation}
		where
		\begin{equation}
			\begin{aligned}
				 & \mathcal{Q}:=\{M: \exists h\neq h'\; M_h = M_{h'} \}\cap \mathcal{D}(V; E) \\
				 & \mathcal{Q}':=\{R: \exists h: R_h = 0 \} \cap \mathcal{D}(V; E').
			\end{aligned}
		\end{equation}
		Then $U$ is dense in $\mathcal{D}(V; E)\times \mathcal{D}(V; E') = \mathcal{E}^{\mathrm{Id}, E, E'}$,
		since $E$ and $E'$ do not have zero fibers.


		Take $(M, R)\in U$.
		There is some $k$ for which $\mathfrak{R}(M, R) \in \mathfrak{R}(\mathcal{E}^{\chi^k, \tilde E^k, \tilde E^{\prime k}})$.
		By Theorem~\ref{thm:identifiability}, for any $\tilde M, \tilde
			R\in \mathrm{End}(V)^H$, a necessary condition for
		$\mathfrak{R}(M, R) = \mathfrak{R}(\tilde M, \tilde R)$ is $|\{\tilde M_h: \tilde R_{h} \neq 0\}|= |H|$.

		Take any $(\tilde M, \tilde R) \in \mathcal{E}^{\chi^k, \tilde E^{k}, \tilde E^{\prime k}}$.
		If $\chi^k$ is not injective, then $|\{\tilde M_h: \tilde R_h \neq 0\}|\le |X^k| < |H|$.
		Thus, suppose $\chi^k$ is injective. If $X_0^k\neq \varnothing$, then
		$|\{\tilde M_h: \tilde R_h \neq 0\}| = |X^k| < |H|$.
		Finally, assume $\chi^k$ is injective and $X_0^k=\varnothing$,
		but $gp = p'$ for distinct points $p, p'\in X^k$ and $g\in G$.
		By Theorem~\ref{thm:identifiability}, there  must exist an $h\in H$ such
		that $\tilde M_{(\chi^k)^{-1}(p)} =  M_h$.
		But then $\tilde M_{(\chi^k)^{-1}(p')} = g\tilde M_{(\chi^k)^{-1}(p)}g^{-1} = M_h = \tilde M_{(\chi^k)^{-1}(p)}$,
		so again $|\{\tilde M_h: \tilde R_h\neq 0\}| < |H|$.
		In other words, we must have $\chi^k \cong \mathrm{Id}$.
		Again by Theorem~\ref{thm:identifiability}, there is a permutation $\nu\in \mathrm{Aut}(H)$
		such that $(M,R)\in \mathcal{E}^{\mathrm{Id}, \tilde E^{k, \nu}, \tilde E^{\prime k, \nu}}$, where
		$\tilde E^{k, \nu}$ denotes the permuted bundle.

		Hence, we have shown
		\begin{equation}
			U \subset \bigcup_{k: \chi^k\cong\mathrm{Id}}^N \bigcup_{\nu\in \mathrm{Aut}(H)}
			\mathcal{E}^{\chi^k, \tilde E^{k, \nu}, \tilde E^{\prime k, \nu}}.
		\end{equation}

		Since $\mathcal{D}(V; E), \mathcal{D}(V; E')$ are irreducible
		(Lemma~\ref{lem:maps-through-bundle-closed}(1)), taking the closure of both
		sides yields
		\begin{equation}
			\mathcal{E}^{\mathrm{Id}, E, E'}\subset
			\bigcup_{k: \chi^k\cong\mathrm{Id}}^N \bigcup_{\nu\in \mathrm{Aut}(H)}
			\mathcal{E}^{\chi^k, \tilde E^{k, \nu}, \tilde E^{\prime k, \nu}}.
		\end{equation}
		Again by irreducibility, $\mathcal{E}^{\mathrm{Id}, E, E'} \subset
			\mathcal{E}^{\chi^k, \tilde E^{k,\nu}, \tilde E^{\prime k,\nu}}$ for
		some $k$ and $\nu$, which is equivalent to $\mathcal{D}(V; E)\subset\mathcal{D}(V; \tilde E^{k,\nu})$
		and $\mathcal{D}(V; E') \subset \mathcal{D}(V; \tilde E^{\prime k,\nu})$.
		By Lemma~\ref{lem:maps-through-bundle-closed}(2),
		$E\hookrightarrow \tilde E^{k, \nu}$ and $E'\hookrightarrow \tilde E^{\prime k, \nu}$.

		If the head dimensions are equal, then $E\cong \tilde E^{k, \nu}$
		and $E'\cong \tilde E^{\prime k, \nu}$
		by maximality of $E$ and $E'$ among $d$-rank-bounded subbundles of $H\times V$.
	\end{proof}
\end{proposition}

%

Finally, we prove Theorem~\ref{thm:components}.
Let $d:= \dim V / |H|$ be the uniform head dimension.
Recall that the (rank-bounded)
equivariance locus $\mathcal{E}$ in the RMHSA parameter space is the union
of the $\mathcal{E}^{\chi, E, E'}$, which are Zariski-closed.
Express this as $\mathcal{E} = \bigcup_{a\in A} \mathcal{E}_a$.
Proposition~\ref{prop:containment-control} identifies a subset $A_0\subset A$, namely those
with $\chi \cong \mathrm{Id}$ and $E, E' \in \mathrm{Sub}_+^{d}(H\times V)$,
with the following property:
there is an $\mathrm{Aut}(H)$-action on $A_0$ such that
if $\mathfrak{R}(\mathcal{E}_a)\subset \bigcup_{c\in C}\mathfrak{R}(\mathcal{E}_c)$ with $a\in
	A_0$ and finite $C\subset A$,
then we can find a $c\in A_0\cap C$ such that $a,c$ are in the same $\mathrm{Aut}(H)$-orbit.

Now, let $(\beta\text{-MHSA})_{\beta\in \mathbf{B}}$ be a finite (otherwise there is nothing to prove) $G$-complete family
of polynomially parameterized MHSA.
By
Corollary~\ref{cor:polynomial-only-one-component}, there is a map $J:
	\mathbf{B} \rightarrow A$ such that $\mathcal{E}^\beta\subset
	\mathcal{E}_{J(\beta)}$. By $G$-completeness, we have $\bigcup_{\beta \in
		\mathbf{B}} \mathfrak{R}(\mathcal{E}_{J(\beta)}) = \bigcup_{a\in
		A}\mathfrak{R}(\mathcal{E}_a)$. That is, for any $a\in A_0$,
we have $\mathfrak{R}(\mathcal{E}_a)\subset \bigcup_{\beta \in \mathbf{B}}\mathfrak{R}(\mathcal{E}_{J(\beta)})$.
Proposition~\ref{prop:containment-control} then
implies there is a $\beta$ such that $J(\beta) \in A_0$ and $a, J(\beta)$ are
in the same $\mathrm{Aut}(H)$-orbit.
That is,
\begin{equation}
	A_0 \subset \mathrm{Aut}(H)\cdot (J(\mathbf{B})\cap A_0),
\end{equation}
or,
\begin{equation}
	|A_0| \le |H|! \cdot |J(\mathbf{B})\cap A_0|
	\le |H|! \cdot |J(\mathbf{B})| \le |H|! \cdot |\mathbf{B}|.
\end{equation}
That is,
\begin{equation}
	|\mathbf{B}| \ge \frac{1}{|H|!} |\mathrm{Sub}_+^d(H\times V)|^2
	= \frac{1}{|H|!} c_{G,d}(V)^{2|H|}.
\end{equation}

\subsection{Derivation of the numerical bound for $G=D_4$}
\label{app:d4_components_bound}

For $G=D_4$ acting on $\mathbb{R}^{C}\otimes \mathbb{R}[D_4]$ 
and eight $C$-dimension attention heads, 
the bound is $8!^{-1} c_{D_4, C}(\rho_{\mathrm{reg}}^{\oplus C})^{16}$.

$c_{D_4, C}(\rho_{\mathrm{reg}}^{\oplus C})$ is the number of $C$-dimensional
subrepresentations up to isomorphism. So
\begin{equation}
    c_{D_4, C}(\rho_{\mathrm{reg}}^{\oplus C})
    = {C+3 \choose 3} + {C+1\choose 3} + \dotsb = \Theta(C^4).
\end{equation}
That is, the bound is asymptotically $8!^{-1} \Theta(C^4)^{16} = \Theta(C^{64})$.
For even $C$, this lower bound is
$$
    \frac{1}{8!}\left[\frac{(C+2)(C+4)(C^2+6C+6)}{48}\right]^{16}.
$$
For $C=96$ (total token dimension $=8C = 768$), it is about $1.6\times 10^{96}$.

%
%
%
%

\section{Equivariant Vector Bundles}
\label{app:equiv-bundle}

Here, we review the notion of equivariant vector bundles, which are used
in Appendix~\ref{app:irrcomp-proof} to decompose the equivariance locus in the
parameter space of RMHSA. For our purpose, it is enough to consider finite
$G$-sets as base spaces.

While we worked with an arbitrary group $G$ in Appendix~\ref{app:proof} and
Theorem~\ref{thm:equiv-attn-char}, we will assume that $G$ is finite in this
section.

\begin{definition}
	\label{def:equiv-bundle}
	A \textbf{$G$-equivariant vector bundle}
	is a $G$-equivariant surjective map $\pi: E \rightarrow X$, where $E, X$ are $G$-sets, $X$ is
	finite, $E_p := \pi^{-1}(p)$ is a finite-dimensional vector space
	for each $p\in X$, and $g: E_p \rightarrow E_{gp}$ is linear.
	\begin{enumerate}
		\item A \textbf{section} is a map $s: X\rightarrow E$ such that $\pi \circ s = \mathrm{Id}_{X}$.
		      The space of sections is denoted by $\Gamma(E)$ and is isomorphic to $\bigoplus_{p\in X} E_p$.
		      It is a $G$-representation through the action $(g\cdot s)(p) = g \cdot s(g^{-1}p)$.
		      For $p\in X$, define the \textbf{evaluation map} as $\mathrm{ev}_p:
			      \Gamma(E)\rightarrow E_p, s \mapsto s(p)$. Identifying $E_p$
		      with the corresponding subspace of $\Gamma(E)$,
		      we may think of $\mathrm{ev}_p$ as belonging to $\mathrm{End}(\Gamma(E))$.
		\item The \textbf{local rank} is $\mathrm{rk}_p(E) := \dim(E_p)$, and depends only
		      on the orbit of $p\in X$.
		\item If $E'$ is another $G$-equivariant vector bundle over $X$, a
		      \textbf{bundle homomorphism} is a $G$-equivariant map $\varphi: E\rightarrow E'$
		      such that $\varphi(E_p)\subset E'_p$ and $\varphi|_{E_p}$ is linear.
		\item A \textbf{subbundle} is a $G$-stable subset $D\subset E$
		      such that $D_p := D\cap E_p$ is a linear subspace of $E_p$ for each $p\in X$.
		\item Given a prescribed function $r: X\rightarrow \mathbb{N}_{\ge0}$,
		      a bundle $\pi: E\rightarrow X$ is said to be \textbf{$r$-rank-bounded}
		      if $\mathrm{rk}_p(E) \le r(p)$ for all $p\in X$.
	\end{enumerate}
\end{definition}

We will always assume that each fiber $E_p$ is equipped with an inner product
such that the induced inner product on $\Gamma(E)$ is $G$-invariant.

Recall that $\mathrm{Sub}^r(E)$ is the set of equivalence classes
of maximal $r$-rank-bounded subbundles of $E$ up to bundle isomorphisms
(see Section~\ref{sec:proof-notations}).

\begin{example}
	For any $G$-set $X$ and any $G$-representation $V$, we can construct the
	\textit{product bundle} $\pi_1: X\times V \rightarrow X$,
	where $G$ acts on both factors in $X\times V$.
\end{example}

\begin{example}
	\label{ex:induced-rep}
	Let $K\subset G$ be a subgroup and let $W$ be a $K$-representation.
	Let $E := G\times_K W = \{[g, w]: g\in G, w\in W\}$
	where $[g, w] = [gk, k^{-1}w]$ for $k\in K$.
	Then $\pi: E\rightarrow G/K, [g, w]\mapsto gK$ is a $G$-equivariant vector bundle,
	and $\Gamma(E)\cong \mathrm{Ind}_K^G(W)$.
	For any $G$-representation $V$, there is a canonical isomorphism
	$\mathrm{Hom}_K(V, W)\cong \mathrm{Hom}_G(V, \Gamma(E))$
	given as follows:
	if $f\in \mathrm{Hom}_K(V, W)$,
	define $\tilde f\in \mathrm{Hom}(V, \Gamma(E))$
	by setting $\tilde f(v)(gK)= gf(g^{-1}v)\in E_{gK}$.
	This is well-defined: if $g' = gk$, then $g'f{g'}^{-1} = gkfk^{-1}g^{-1} = gfg^{-1}$
	because $f$ is $K$-equivariant.
	It is easy to check that the map $\tilde f$ is $G$-equivariant.
	Similarly, we also have adjunction in the other direction $\mathrm{Hom}_K(W, V)\cong \mathrm{Hom}_G(\Gamma(E), V)$.
\end{example}

\begin{lemma}
	\label{lem:eval-map-equiv}
	The map $X \rightarrow \mathrm{End}(\Gamma(E)), p \mapsto \mathrm{ev}_p$ is $G$-equivariant.
	\begin{proof}
		Take a section $s\in \Gamma(E)$. Then
		\begin{equation}
			\begin{aligned}
				 & \mathrm{ev}_{gp} (s) = s(gp) = gg^{-1} s(gp)
				= g (g^{-1}\cdot s)(p)                          \\
				 & = g \cdot \mathrm{ev}_{p} (g^{-1}\cdot s)
				= [g \mathrm{ev}_p g^{-1}] (s).
			\end{aligned}
		\end{equation}
		Since $s$ was arbitrary, we have $\mathrm{ev}_{gp} = g\mathrm{ev}_p g^{-1}$.
	\end{proof}
\end{lemma}

The motivation to consider equivariant vector bundles is the following lemma:
\begin{lemma}
	\label{lem:equiv-system-rank-bounded-existence}
	Let $M: X\rightarrow \mathrm{End}(V), p \mapsto M_p$
	and $r: X\rightarrow \mathbb{N}_{\ge 0}$
	an arbitrary function. The following are equivalent:

	\begin{enumerate}
		\item[(1)] The map $p\mapsto M_p$ is $G$-equivariant, and $\mathrm{rank}(M_p) \le r(p)$ for all $p\in X$.
		\item[(2)] There exists a maximal
		      $r$-rank-bounded subbundle $E \in \mathrm{Sub}^{r}(X\times V)$
		      and equivariant linear maps $\psi_1\in\mathrm{Hom}_G(V, \Gamma(E)), \psi_2\in \mathrm{Hom}_G(\Gamma(E),V)$
		      such that $M_p = \psi_2\mathrm{ev}_p\psi_1$ for all $p\in X$.
	\end{enumerate}

	More concisely (this is Eq.~\eqref{eq:equiv-system-rank-bounded}),
	\begin{equation}
		\mathcal{D}_r(V; X)= \bigcup_{E \in \mathrm{Sub}^r(X\times V)} \mathcal{D}(V; E).
	\end{equation}

	\begin{proof}
		We may assume $X = G/K$ is a homogeneous space. The general statement is
		recovered by considering each orbit separately. In this case, $r$ is just a
		nonnegative integer.
		\;\\
		(1) $\Rightarrow$ (2)\\
		Let $p_0\in X = G/K$ be the coset of the identity of $G$ (i.e. $p_0=K$ but
		we use the notation $p_0$ to emphasize that we think of it as a point).
		Then $M_{p_0} \in \mathrm{End}_K(V)$.
		Let $W \subset V$
		be a maximal $K$-subrepresentation of dimension at most $r$
		containing $\mathrm{im}(M_{p_0})$.
		We can then find $f_1\in \mathrm{Hom}_K(V, W), f_2\in \mathrm{Hom}_K(W, V)$
		such that $f_2 f_1 =M_{p_0}$.
		Let $E:=G\times_K W$ as in Example~\ref{ex:induced-rep}.
		Set $\psi_2 = \tilde f_2\in \mathrm{Hom}_G(\Gamma(E), V)$ and $\psi_1 = \tilde f_1\in \mathrm{Hom}_G(V, \Gamma(E))$ (see
		Example~\ref{ex:induced-rep}), for which $M_{p} = \psi_2\mathrm{ev}_p
			\psi_1$.
		Moreover, $\mathrm{rk}(E) = \mathrm{dim}(W)
			\le r$,
		so $E\in \mathrm{Sub}^r(X\times V)$.
		\;\\[.5em]
		(2) $\Rightarrow$ (1)\\
		By Lemma~\ref{lem:eval-map-equiv}, we have $\psi_2\mathrm{ev}_{gp} \psi_1 = \psi_2 g\mathrm{ev}_p g^{-1} \psi_1 = g\psi_2\mathrm{ev}_p \psi_1g^{-1}$.
	\end{proof}
\end{lemma}

\begin{lemma}
	\label{lem:maps-through-bundle-closed}
	\leavevmode
	\begin{enumerate}
		\item[(1)] $\mathcal{D}(V; E)$ is Zariski-closed in $\mathrm{End}(V)^X$ and irreducible.
		\item[(2)] If $E, E'$ are subbundles of $X\times V$ and $\mathcal{D}(V; E')\subset \mathcal{D}(V; E)$, then
		      $E'$ is isomorphic to a subbundle of $E$.
	\end{enumerate}
	\begin{proof}
		\;\\
		(1)\\
		$\mathcal{D}$ is irreducible because it is the image of an affine space
		under a polynomial map.

		We may assume $X = G/K$ is a homogeneous space; the general case follows
		from considering each orbit separately and
		the fact that the Cartesian product of two Zariski-closed sets is closed.

		Let $W := E_{K}$ be the fiber above the coset of the identity. Consider the map
		\begin{equation}
			\begin{aligned}
				 & \delta: \mathrm{End}_K(V) \rightarrow \mathrm{End}(V)^X \\
				 & M \mapsto (gMg^{-1})_{gK\in X}.
			\end{aligned}
		\end{equation}
		The map $\delta$ is a biregular isomorphism from $\mathrm{End}_K(V)$ onto
		$\delta(\mathrm{End}_K(V))$, and $\delta(\mathrm{End}_K(V))\subset \mathrm{End}(V)^X$ is closed.
		Moreover, we have $\delta(D) = \mathcal{D}$, where
		$D := \mathrm{Hom}_K(W, V)\mathrm{Hom}_K(V, W)\subset \mathrm{End}_K(V)$.
		It is easy to check that $D$ is closed (decompose $V, W$ into irreps of $K$, then $D$
		consists of exactly those equivariant endomorphisms
		with per-irrep rank constraints).
		Therefore, also $\mathcal{D} = \delta(D)$ is closed in $\mathrm{End}(V)^X$.

		(2)\\
		Again, we may assume $X = G/K$. Let $W, W'$ be the fibers over $K$ of $E$ and $E'$ respectively.
		Since $E, E' \subset X\times V$, the fibers $W, W'$ are subrepresentations of $\mathrm{Res}^G_K(V)$.
		Then $\mathcal{D}(V;E')\subset \mathcal{D}(V;E)$ is equivalent to
		$\mathrm{Hom}_K(W',V)\mathrm{Hom}_K(V, W')\subset \mathrm{Hom}_K(W,
			V)\mathrm{Hom}_K(V, W)$. Clearly, this can only happen if $W'$ embeds $K$-equivariantly
		into $W$, from which it follows that $E'$ is isomorphic to a subbundle of $E$.
	\end{proof}
\end{lemma}

\section{Experiment Details}

\subsection{Teacher-Student}
\label{app:exp-ts}

For the teacher-student experiment involving pairs of $D_4$-equivariant attention
layers as well as the unconstrained (vanilla) attention shown in
Fig.~\ref{fig:teacher-student}, the hyperparameters are as follows:
\begin{itemize}
	\item feature space $\mathbb{R}^{\Cin} = \mathbb{R}^{\Cout} = \mathbb{R}^8$ with the regular representation ($\mathbb{R}[D_4]$)
	\item number of heads $H=2$
	\item head dimension $d=4$
	\item token count $L=10$
	\item total trainin steps $=4000$
	\item learning raite: Cosine Annealing LR with $\eta_{\text{max}}=10^{-2}$ and $\eta_{\text{min}}=10^{-4}$
	\item training bach size $B_{\text{train}} = 64$
\end{itemize}

In total, we implement six 2-head equivariant attention architectures with four-dimensional heads.
According to Theorem~\ref{thm:equiv-attn-char}, these are indexed by (1) a $G$-action on the heads,
and (2) for each orbit $O\subset \{1, 2\}$, a maximal subreprsentation of $\mathrm{Res}^{D_4}_{\mathrm{Stab}(p_O)}(\mathbb{R}[D_4])$
of dimension at most four, where $p_O$ is a chosen base point and $\mathrm{Stab}(p_O)$ is the stabilizer subgroup.
In our case, the maximal subrepresentations all happen to saturate the dimension constraint (they are all four-dimensional).
We consider two actions:
\begin{itemize}
	\item $D_4/D_4\sqcup D_4/D_4$: the trivial action (two orbits each consisting of one point).
	\item $D_4/C_4$: for this action, $s$ (reflection) swaps the two heads.
\end{itemize}
The architectures for the trivial action are then indexed by a pair of four-dimensional
subrepresentations of $\mathbb{R}[D_4]$.
For $D_4/C_4$, they are indexed by four-dimensional subrepresentations of
$\mathrm{Res}^{D_4}_{C_4}[\mathbb{R}[D_4]] = \mathbb{R}^2\otimes \mathbb{R}[C_4] \cong 2\mathrm{A}\oplus 2\mathrm{B} \oplus 2\mathrm{E}_1$.

The numbers reported in Fig.~\ref{fig:teacher-student}'s heatmap are relative MSE loss, defined as

\begin{equation}
	l(t, s) := \frac{1}{E_t} \frac{1}{B_{\text{eval}}LC} \sum_{b,l,c}
	[t(x)_{b,l,c} - s(x)_{b,l,c}]^2,
\end{equation}
where $t$ is the teacher, $s$ is the student, and
$E_t := \frac{1}{B_{\text{eval}}LC}\sum_{b,c,l} t(x)_{b,l,c}^2$.
Here, the $x$ is a batch of $B_{\text{eval}} = 4096$ random iid Gaussian token sequences, fixed across all runs.

We run five seeds for each teacher-student pair and report the the smallest $l(s,t)$ across the five runs.


\subsection{Modular Addition}
\label{app:exp-modadd}

We train a single layer of unconstrained (non-equivariant) attention to learn the map
\begin{equation}
	(a, b, c) \mapsto (b+c, c+a, a+b),
\end{equation}
where addition is defined modulo $p$ with $p=24, 15, 17$ (with the results for $p=17$ shown below). For example,
$(20, 4, 13)\mapsto (17, 9, 0)$ for $p=24$.

The input is one-hot encoded in $\mathbb{R}^p$, and the layer's output is centered. That is,
the component along the direction $(1, \dotsc, 1)^\top\in \mathbb{R}^{p}$ is projected out.
The group $C_{24}$ acts on $\mathbb{R}^{\Cin} = \mathbb{R}^{\Cout} =\mathbb{R}^{p}$ by cyclically
permuting the entries.
The action on the output space is ``doubled'', in the sense that $\rhoout(g) = \rhoin(g)^2$.

For $p=24$, there are $5{,}760$ training examples, each consisting of a pair of triples $((a, b, c), (b+c, c+a, a+b))$.
Each test item is a triple $((a, b, c), j, (b+c, c+a, a+b)_j)$, for example $((20, 4, 3), 0, 7)$, where the model
is asked to predict the $j$-th pairwise sum of the input. There are $3{,}456$ test items in total.
For $p=15$ the number of training tuples and the number of test items are $1{,}395$ and $1{,}080$.
For $p=17$, this numbers are $2{,}261/1{,}496$.

The attention layer has $H=32$ heads, each with head dimension $d= 8$.
It is trained by AdamW with $\beta = (0.9, 0.98)$ and weight
decay $1.0$, with each step consisting of one single batch with all examples.
The $p=24$ model is trained for $6{,}000$ steps, while the $p=15$ and $p=17$
models are trained for $40{,}000$ steps.

The learning rate linearly increases to $10^{-3}$ within the first $10$ epochs,
then stays constant afterwards.
The total loss is $l_{\text{task}} + l_{\text{eq}}$, where $l_{\text{task}}$ is cross-entropy for $24$ classes,
and the equivariance loss is calculated as follows:
in each training step, $B=64$ triples of i.i.d. Gaussian tokens are drawn from $\mathcal{N}(0, \mathbbm{1})$.
I.e., a random $x\in \mathbb{R}^{B\times 3 \times p}$ with each entry i.i.d. from $\mathcal{N}(0, 1)$.
For each $b\in \{0, \dotsc, B-1\}$, randomly draw
$S=24$ group elements $g_{b}^{(s)}\in C_{p}$, $s=\{0, \dotsc, S-1\}$.
Then set
\begin{equation}
	l_{\text{eq}} :=
	\frac{\frac{1}{B(S-1)} \sum_{b,s} \lVert z_{bl}^{(s)} - \bar z_{bl}\rVert^2}{\frac{1}{BS-1} \sum_{b,s} \lVert z_{bl}^{(s)} - \bar z_{l}\rVert^2},
\end{equation}
where
\begin{equation}
	z_{bl}^{(s)}
	:= \rhoout(g_b^{(s)})^{-1}
	F_j(\rhoin(g_b^{(s)}) x_b),
	\qquad
	\bar z_{bl} := \frac{1}{S}\sum_{s}z_{bl}^{(s)}, \quad
	\bar z_l := \frac{1}{BS} \sum_{b,s} z_{bl}^{(s)}.
\end{equation}

For $p=24$, test accuracy reaches $100\%$ at all three output token positions at step $6{,}000$.
While for $p=15$ the test accuracy is still $0\%$ at step $40{,}000$.

For $p=17$, the final test accuracy is $95\%$.
In Fig.~\ref{fig:modadd_p17_2d}, we carry out the same analysis as we did
for $p=24$, this time for both $M$ and $R$ but projecting only on irrep planes.
The irreps of $C_{17}$ are $W_0$ (trivial) and $W_{\omega}, 1\le \omega \le 8$ (2D),
and $\mathbb{R}^{17}$ again decomposes as $W_0\oplus W_1\oplus \dotsb \oplus W_8$.
For the $R_h$, the output frequencies are doubled, but we still label with the
original frequencies. So, for the $R_h$, the subspace $W_{\omega_i}\otimes W_{\omega_o}$
contains of $W_{\omega_i-2\omega_o}$ and $W_{\omega_i+2\omega_o}$.

This time, the $32$ heads organized themselves into one $C_{17}$-orbit of size $17$,
one ``pseudo-orbit'' of size $8$, and two singleton orbits.

The size-$17$ orbit actually consists of $22$ heads, $10$ of each paired up to form $5$ two-head clusters.
If we sum the respective pairs $(R_h, R_{h'})$, we see in the bottom-left panels of Fig.~\ref{fig:modadd_p17_2d} that they do land on a size-$17$ orbit.
However, this description is not exact:
none of the head-pairs share the same QK matrix.
In fact, the distance between each pair of heads
in the $W_2$ plane of Fig.~\ref{fig:modadd_p17_2d}
seems to be constant across all five pairs.

Also note that the family of eight heads (orange panels) do \emph{not} form an actual orbit of $C_{17}$
(any orbit must necessarily have $1$ or $17$ elements).
Nonetheless, they still form regular polygons in the projected irrep planes, similarly to the square pseudoorbit in the $p=15$ run (Fig.~\ref{fig:modadd_p15}.
We believe that this can be explained as follows:
even though the octagon is not stable
under rotation by $2\pi\cdot 7/17$ (which would be necessary for equivariance by Theorem~\ref{thm:equiv-attn-char}),
it is approximately a circle, which \emph{is} stable under any rotation.
So, the octagonal arrangement might be a local minimum for the equivariance
loss, and once the eight heads formed such structure, finding a ``better one''
might require the model to sacrifice a lot of equivariance and//or task accuracy
first.
This explanation is not as convincing for the square, but the $W_3$
where this happens has very little overlap (Frobenius) with the $R_h$.

\begin{figure}[hthbp]
	\centering
	\includegraphics[width=.98\textwidth]{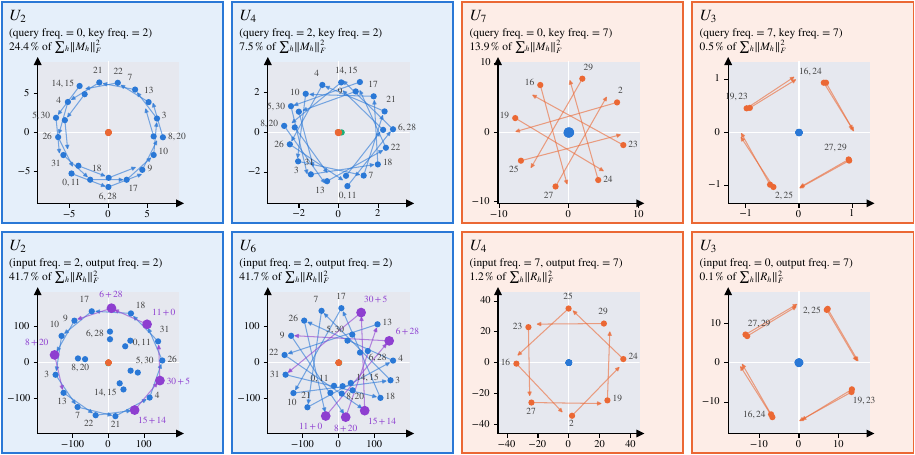}
	\caption{$M_h$ (top row) and $R_h$ (bottom row) for the $p=17$ run.
	In the first two $R$ plots, we sum $R_h+R_{h'}$ for two heads if the respective $M_h, M_{h'}$ are close. The sum is then shown
	as one large purple dot.
	}
	\label{fig:modadd_p17_2d}
\end{figure}

\subsection{Vector Fields}
\label{app:exp-vf}
\subsubsection{Data}
We generate vector fields $\mathbb{R}^2\rightarrow \mathbb{R}^2$ of the form
\begin{equation}
	\mathbf{v}(\mathbf{x}) =
	\sum_{a=1}^{M}
	\mathbf{u}_a \cos(\mathbf{k}_a\cdot \mathbf{x} + \phi_a)
\end{equation}
where:
\begin{itemize}
	\item $\phi_a$ is a random phase.
	\item $\mathbf{u}_a$ is a random amplitude vector with each component drawn from
	      $\mathcal{N}(0, \sigma^2=0.06)$.
	\item each wave vector $\mathbf{k}_a$ is defined to be $2\pi\lambda_a^{-1} \mathbf{s}_a$,
	      where $\lambda_a$ is a wavelength uniformly sampled from $[64, 160]$
	      and $\mathbf{s}_a$ is a random unit vector.
\end{itemize}

In our experiments, we take $M=6$ and generate $40{,}000$ training,
$1{,}000$ validation, and $1{,}000$ testing vector fields.
Each generated vector field $\mathbf{v}$ is represented by a
$256\times 256$ image with one arrow drawn on each $16\times 16$ patch representing the
value of the vector field at the center of the patch.
We allow the arrows to extend to nearby patches,
just to make the task more difficult.
Random color perturbations, both in the drawn arrows and in the background, are
added to each image (see Fig.~\ref{fig:vf-images}).

\begin{figure}[htbp]
	\centering
	\includegraphics[width=.8\textwidth]{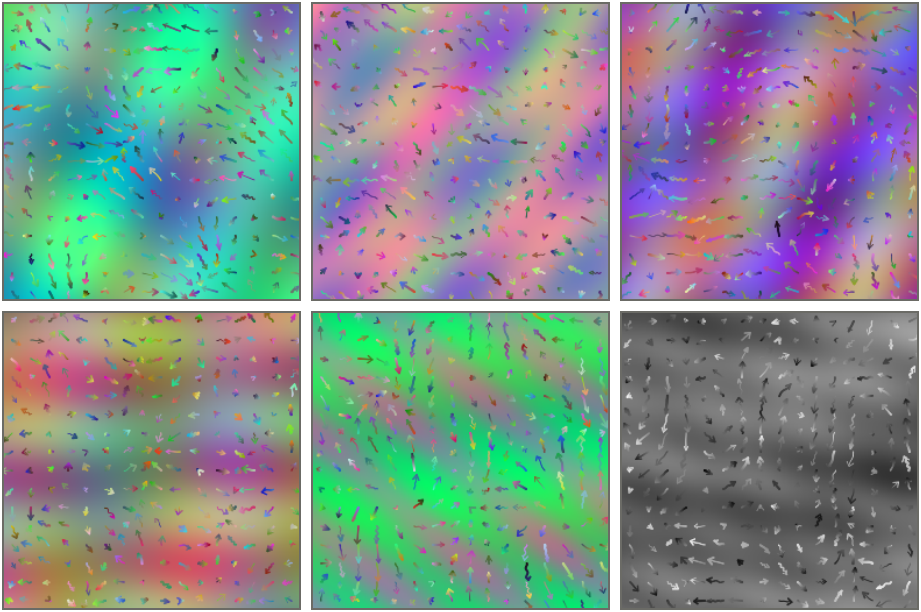}
	\caption{Six training images for the vector field task.}
	\label{fig:vf-images}
\end{figure}

The target features to be learned by the model are the two components of the vector field
$(v_x, v_y)$ and the four components of the derivative $\partial_\mu v_\nu$, where $\mu, \nu \in \{x, y\}$, evaluated at the center of each $16\times 16$ patch,
i.e., at $\{-120, -104, -88, \dotsc, 88, 104, 120\}^2 \subset \mathbb{R}^2$. In our dataset, we already decompose these four components into
the four 1D irreps of $D_4$, each of which has a clear physical meaning (see
Table~\ref{tab:vector-field-targets}). However, whether the features are
decomposed into irreps does not matter mathematically since mapping $\nabla
	\mathbf{v}$ to the four irrep components is linear, which can be learned by the
attention layer.
The task of mapping an input image to $(\mathbf{v}, \nabla\mathbf{v})$
is $D_4$-equivariant.

\begin{table}[htbp]
	\centering
	\begin{tabular}{|c|c|c|}
		\hline
		target         & formula                         & irrep          \\
		\hline
		\hline
		vector field   & $(v_x, v_y)$                    & $\mathrm{E}_1$ \\
		\hline
		divergence     & $\partial_xv_x + \partial_yv_y$ & $\mathrm{A}_1$ \\
		\hline
		curl           & $\partial_xv_y - \partial_yv_x$ & $\mathrm{A}_2$ \\
		\hline
		axial shear    & $\partial_xv_x - \partial_yv_y$ & $\mathrm{B}_1$ \\
		\hline
		diagonal shear & $\partial_xv_y + \partial_yv_x$ & $\mathrm{B}_2$ \\
		\hline
	\end{tabular}
	\caption{The target features.}
	\label{tab:vector-field-targets}
\end{table}

\subsubsection{Model and Training}
We partition our model into two parts:
\begin{itemize}
	\item A \textit{backbone}, which is
	      a $11$-layer $D_4$-equivariant transformer backbone with $288$ channels, consisting
	      of $36$ copies of the regular representation. The MLP ratio is $4$, with
	      the hidden layer containing $144$ copies of the regular representation.
	      Each attention layer is an \verb|EquivariantCoupledAttention| \citep{ongUnifiedFrameworkVision2026} with six heads.
	\item A vanilla (unconstrained) multi-head attention
	      with $288$ channels and six heads, followed by a linear $288\rightarrow 6$
	      readout.
	      We define the OV matrices to include this readout, so each $R_h$
	      is a $6\times 288$ matrix.
\end{itemize}

The model is trained for $200$ epochs at batch size $64$
with MSE loss of the six target entries.
The optimizer is AdamW with $\beta = (0.9, 0.95)$, weight decay $0.05$,
and gradient clipping at norm $1.0$.
The learning rate linearly increases
over the first $5\%$ of the steps to $5\times 10^{-4}$,
then cosine-decays to $10^{-5}$ at epoch $150$, then stays
constant until epoch $200$.

We perform random per-image $D_4$-augmentation and apply random color jitter,
random color channel permutation, as well as Gaussian noise.

\subsubsection{Norms}
The norms  used in the analysis presented in
Section~\ref{sec:equivvf} are not the standard Euclidean norms. This is because
the final attention layer only ever sees what the backbone outputs, and it has
no reason to be equivariant in directions where the outputs have no component.

The norms are defined as follows.
Suppose the backbone (blue part of the architecture in Fig.~\ref{fig:vf}) outputs $Y\in \mathbb{R}^{B_{\mathrm{eval}}L\times C}$ on test data.
Let $\mu_c:= \frac{1}{B_{\mathrm{eval}}L} \sum_{bl} Y_{bl, c}$ denote the mean token feature, and
let $\Sigma := \frac{1}{B_{\mathrm{eval}}L}(Y-\mu)^\top(Y-\mu) \in \mathbb{R}^{C\times C}$ be the covariance.
Define a positive semidefinite bilinear form on $\mathrm{Mat}_{C\times C}(\mathbb{R})$ by
\begin{equation}
	\braket{M, M'}_Y :=
	\mathrm{tr}\left(\sqrt{\Sigma}M^\top \Sigma M'\sqrt{\Sigma}\right)
	+  \mu^\top M^\top \Sigma M' \mu.
\end{equation}

This is chosen so that
\begin{equation}
	\lVert M\rVert_Y^2 = \mathbb{E}_{x\sim \mathcal{D}} \mathrm{Var}_{y\sim \mathcal{D}} x^\top M y,
\end{equation}
where $\lVert M \rVert_Y^2 :=  \braket{M, M}_Y$, and $\mathcal{D}$ denotes the empirical single-token distribution.
This measures how much the raw attention scores varies for a given query token, averaged over all query token.

Similarly, on $\mathbb{R}^{N\times C}$, define
\begin{equation}
	\braket{R, R'}_Y := \mathrm{tr} \left(R^\top R'\Sigma\right) + \mu^\top R^\top S \mu,
\end{equation}
which satisfies
\begin{equation}
	\lVert R\rVert_Y^2 = \mathbb{E}_{x\sim \mathcal{D}}  \lVert Rx \rVert^2.
\end{equation}
The orthogonal projection used in panel (f) of Fig.~\ref{fig:vf}
is with respect to $\braket{,}_D$.

\end{document}